\documentclass{article}

\PassOptionsToPackage{numbers, compress}{natbib}
\usepackage{arxiv}
\usepackage{tikz}
\usepackage{multirow}
\usetikzlibrary{arrows.meta,positioning,calc,decorations.pathreplacing}

\usepackage{amsthm}
\usepackage{algorithm}
\usepackage{algpseudocode}
\usepackage{amsthm}
\newtheorem{lemma}{Lemma}
\usepackage{amssymb}
\usepackage{xspace}
\newcommand{\ie}{i.e.\xspace}

\newcommand{\QuasiMoTTo}{QuasiMoTTo\xspace}
\newcommand{\qmt}{\QuasiMoTTo}
\newtheorem{theorem}{Theorem}
\newtheorem{proposition}{Proposition}
\usepackage{graphicx}

\usepackage{algpseudocode}  % \Require, \Ensure, \State, \Comment, \Return, \If, \For, etc.
\usepackage{amsmath,amsfonts,bm}

\def\eqref#1{equation~\ref{#1}}
\def\1{\bm{1}}

\def\mL{{\bm{L}}}

\def\mR{{\bm{R}}}

\DeclareMathAlphabet{\mathsfit}{\encodingdefault}{\sfdefault}{m}{sl}
\SetMathAlphabet{\mathsfit}{bold}{\encodingdefault}{\sfdefault}{bx}{n}

\newcommand{\R}{\mathbb{R}}

\newcommand{\bx}{\mathbf{x}}

\newcommand\blfootnote[1]{%
  \begingroup
  \renewcommand\thefootnote{}\footnote{#1}%
  \addtocounter{footnote}{-1}%
  \endgroup
}

\newcommand{\Dec}{\operatorname{Dec}}
\usepackage[utf8]{inputenc} % allow utf-8 input
\usepackage[T1]{fontenc}    % use 8-bit T1 fonts
\usepackage{hyperref}       % hyperlinks
\usepackage{url}            % simple URL typesetting
\usepackage{booktabs}       % professional-quality tables
\usepackage{amsfonts}       % blackboard math symbols
\usepackage{nicefrac}       % compact symbols for 1/2, etc.
\usepackage{microtype}      % microtypography
\usepackage{xcolor}         % colors

\definecolor{cGray}    {HTML}{8FA3C2}  \definecolor{cGrayD}   {HTML}{2C3E5C}
\definecolor{cAmber}   {HTML}{F5C04A}  \definecolor{cAmberD}  {HTML}{8B6914}
\definecolor{cGreen}   {HTML}{4FB286}  \definecolor{cGreenD}  {HTML}{1F5A40}
\definecolor{cTerra}   {HTML}{EE6C5C}  \definecolor{cTerraD}  {HTML}{8B2D1F}
\definecolor{cViolet}  {HTML}{B89BC9}  \definecolor{cVioletD} {HTML}{4F356A}
\definecolor{cInk}     {HTML}{1A1A1A}
\definecolor{cMid}     {HTML}{555555}
\definecolor{cAxis}    {HTML}{444444}
\definecolor{cU}       {HTML}{3D5A80}
\definecolor{cGap}     {HTML}{B0361C}

\newcommand{\iid}{i.i.d.\@\xspace}

\title{
QuasiMoTTo: Quasi-Monte Carlo Test-Time Scaling
}

\author{%
  Michael Y. Li\textsuperscript{*} \\
  Stanford University \\
  \And
  Anthony Zhan\textsuperscript{*} \\
  Stanford University \\
  \AND
  Kanishk Gandhi \\
  Stanford University \\
  \And
  Noah D. Goodman \\
  Stanford University \\
  \And
  Emily B. Fox \\
  Stanford University \\
}

\usepackage{amsthm}
\begin{document}

\maketitle

\begin{abstract}
Scaling inference compute, by generating many parallel attempts per problem, is a costly but reliable lever for improving language model capabilities.
By default these attempts are generated independently, wasting inference compute on redundant solutions.  
This waste seems unavoidable.
After all, independence is what makes parallel sampling trivial to scale.
However, this tradeoff is not fundamental: there is a rich design space of samplers that generate \emph{correlated but exact} samples entirely in \emph{parallel}. 
We explore this design space as an avenue for improving sample efficiency in scaling inference compute and reinforcement learning (RL). 
Concretely, we introduce \textbf{QuasiMoTTo}, which uses correlated samples as a \emph{drop-in replacement} for i.i.d. samples.
To generate these samples, \qmt uses a reparameterization of autoregressive sampling as inverse-CDF sampling and draws the underlying uniforms with quasi-Monte Carlo (QMC); because QMC spreads the uniforms out more evenly than i.i.d., the resulting samples cover the output space with far less redundancy. 
Even though the batch is correlated, each sample is marginally distributed according to the language model, so we can use the batch for policy-gradient training.
Our empirical analysis focuses on understanding how efficiently QuasiMoTTo can turn compute into performance.
To evaluate correlated samplers, whose dependence breaks standard pass@$k$ estimators, we first develop an unbiased bootstrap estimator.
Across four reasoning benchmarks, QuasiMoTTo matches i.i.d.~pass@$k$ accuracy with 25--47\% fewer samples.
Strikingly, QuasiMoTTo often saturates an upper bound on pass@$k$ that holds for any marginal-preserving sampler.
We also apply QuasiMoTTo to policy-gradient RL (GRPO) where it matches i.i.d. performance with 50\% fewer training steps.
These gains come from higher coverage, which yields a stronger learning signal per batch.
\end{abstract}
\blfootnote{\textsuperscript{*} Equal contribution. Correspondence to \texttt{michaelyli@stanford.edu}.}

\section{Introduction}
\input{figure_tex/method_illustrative}
Parallel sampling, generating $k$ independent (i.i.d.) attempts per prompt, is a fundamental primitive in inference compute scaling and reinforcement learning (RL).
But independence is wasteful: many samples in a batch visit the same high-probability regions. 
This redundancy is costly. 
In test-time scaling, ideally each sample would explore a different approach, but independent samples rediscover the same solutions; therefore, finding a correct solution to a hard problem can require a very large $k$.
In RL, the same redundancy can force large group sizes because methods like GRPO~\citep{guo2025deepseekr1} need within-group reward variation to produce a learning signal.
Reducing this redundancy can mitigate the steep cost of methods that scale inference compute, which are driving tremendous progress in language model (LM) capabilities~\citep{openai2026openaio1card}.

Intuitively, we could reduce redundancy by \emph{correlating} these samples so that they depend on each other. 
We could correlate these samples sequentially — conditioning each draw on the previous ones — but this forfeits the parallelism that makes independent sampling simple to scale.
Parallel alternatives either distort the LM's next-token distribution~\citep{vijayakumar2018diverse,pmlr-v97-kool19a}, biasing the estimates RL depends on, or exploit its sensitivity to the prompt~\citep{novikov2025alphaevolvecodingagentscientific}, where any increase in coverage is incidental rather than guaranteed.

We observe that there is a rich design space of sampling methods that can generate correlated samples that decode completely in \emph{parallel} yet remain exact samples from the LM~\citep{vilnis2023}.
We explore this design space in \textbf{\qmt}, which replaces \iid samples with these dependent samples in inference compute scaling and RL. 
To do this, \qmt leverages Monte Carlo methods and coding theory, building on recent work on arithmetic sampling~\citep{vilnis2023,parashar2025quasirandommultisampleinferencelarge}. 
Instead of drawing samples independently, we generate samples that repel one another to cover the output distribution more efficiently~\citep{owen2013monte}.
Concretely, we first use randomized Quasi-Monte Carlo (QMC) to generate $k$ samples in the unit interval that are more evenly spread than \iid uniform samples. 
We then map these to $k$ LM samples via arithmetic coding~\citep{mackay2002}.
Even though samples are coordinated to boost coverage, the procedure is both embarrassingly \emph{parallel} and \emph{exact}: each rollout is distributed exactly as if it had been sampled \iid from the LM.
Generation is parallel because the arithmetic coding parameterization of autoregressive sampling reduces sampling a sequence to pushing a single uniform through the LM's inverse CDF: once we draw a batch of uniforms upfront, each rollout decodes independently, with no communication across rollouts.
The procedure produces exact LM samples because QMC guarantees each sample marginally uniform, even though samples are correlated.
This ensures inverse-CDF sampling~\citep{wasserman2010statistics} produces an exact draw from the LM.

The goal of using dependent samples is to improve sample efficiency: fewer samples give the same performance.
We therefore evaluate \qmt with a focus on understanding whether we can generate fewer samples from an LM to get the same performance as i.i.d. sampling in test-time scaling and RL.
\textbf{(1) Test-time scaling.}
To measure how much more coverage \qmt provides for fixed $k$, we compare its pass@$k$ on reasoning tasks relative to i.i.d. sampling. 
To characterize the best possible performance, we compare against a union-bound ceiling reflecting the best pass@$k$
achievable by any marginal-preserving sampler.
Strikingly, on reasoning benchmarks, \qmt nearly saturates this ceiling, leaving little room for any marginal-preserving sampler to perform better.
\textbf{(2) Reinforcement learning.} Because our sampler produces exact samples from the LM, it integrates directly into RL pipelines like GRPO. 
We show that using \qmt reduces the number of training steps required to achieve a target performance relative to i.i.d. sampling.

\section{\qmt}
\label{sec:method}

\subsection{Setup}
\label{sec:setup}
\begin{figure}[t]
\centering
\includegraphics[width=0.5\linewidth]{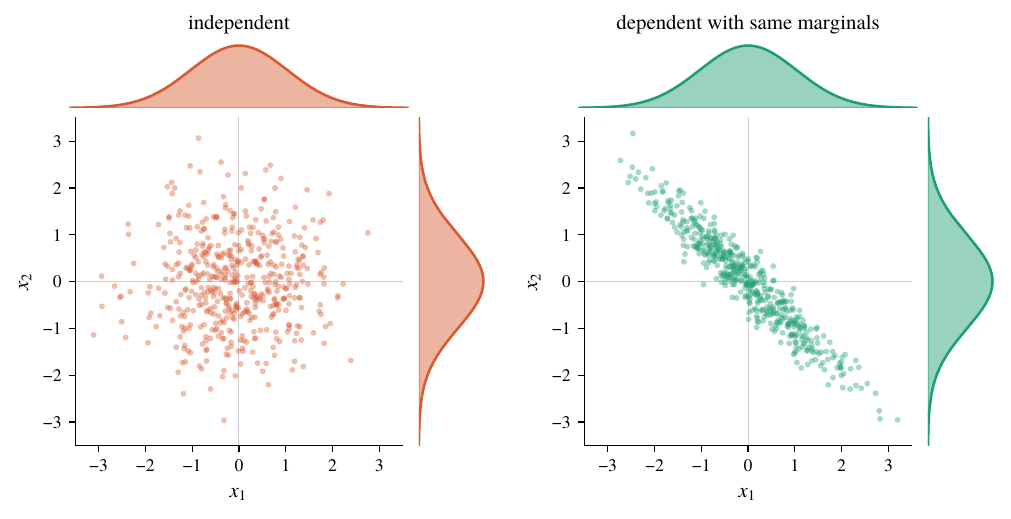}
\caption{\textbf{Same marginals, different joints.} Two distributions over $(x_1, x_2)$ with identical marginals (see histograms) but different dependence structure: independent (left) versus negatively correlated (right). 
Since we often only need to preserve marginals, we can exploit this flexibility to generate LM rollouts that are correlated in a way that better covers the output space.}
\label{fig:marginals_vs_joint}
\end{figure}
A common motif in language modeling is drawing a batch of $k$ samples
$\{\tau_i\}_{i=1}^k$ from a model $\pi_\theta$ and aggregating over them. GRPO
estimates the policy gradient by averaging over a sampled group of responses;
test-time scaling produces a final answer by combining parallel attempts. 
In
each case the batch is a finite-sample approximation to $\pi_\theta$, and we
want it to be (1) \emph{marginally correct}---each $\tau_i$ distributed
according to $\pi_\theta$---and (2) \emph{high-coverage}---spread across the
model's support rather than clumped on a few modes.
The key observation is that correctness asks only for the right marginals, which leaves the joint free; there is a large space of samplers that design the joint distribution in a way that improves coverage.

In particular, marginal correctness is all that average-type estimators need. 
Let $\mu$ be the
joint law of the batch $(\tau_1,\dots,\tau_k)$ and $\mu_i$ its $i$-th marginal.
For any per-rollout quantity $h$,
\[
\mathbb{E}_\mu\!\left[\frac{1}{k}\sum_{i=1}^k h(\tau_i)\right]
= \frac{1}{k}\sum_{i=1}^k \mathbb{E}_{\tau_i \sim \mu_i}\!\left[h(\tau_i)\right]
= \mathbb{E}_{\tau \sim \pi_\theta}\!\left[h(\tau)\right],
\]
where the first equality is linearity of expectation and holds for \emph{any}
joint $\mu$, and the second uses only the marginal condition $\mu_i = \pi_\theta$.
The vanilla policy-gradient estimator with no baselines ($h(\tau) = \nabla\log\pi_\theta(\tau)\,R(\tau)$)
is one such average.
A marginally correct sampler is therefore a drop-in replacement for
i.i.d.\ in every estimator of this form.
The i.i.d.\ default corresponds to the product law
$\mu = \pi_\theta^{\otimes k}$, but that is only one joint among the many that
share the marginals $\mu_i = \pi_\theta$. 
Two distributions can agree on their
marginals and still have substantially different joint distributions---Figure~\ref{fig:marginals_vs_joint} illustrates this for Gaussians.
Another intuitive example is an antithetic pair. 
If we draw $u \sim \mathrm{Unif}[0,1]$, then the tuple $(u, 1-u)$ has uniform marginals, by symmetry, but the coordinates are negatively correlated.

We can exploit this flexibility to design the joint distribution in a way that boosts coverage.
This can often have beneficial effects on post-training. 
For example, policy gradient-based RL methods like GRPO benefit from higher coverage, since there is only a non-zero gradient signal when at least one of the attempts in a group is correct.

The construction has two parts: a randomized
Quasi-Monte Carlo coupling on $\mathrm{Unif}[0,1]$ that induces the dependence
(\S\ref{sec:primer}), and arithmetic coding~\citep{vilnis2023, mackay2002} that maps those uniforms to exact
samples from $\pi_\theta$, preserving the marginals by construction
(\S\ref{sec:inverse-cdf}).

\subsection{Generating dependent samples via Quasi-Monte Carlo}
\label{sec:qmc}
Randomized Quasi-Monte Carlo (QMC) is a technique to generate dependent samples with target marginals. 
Instead of $k$ independent points, it constructs $k$
\emph{low-discrepancy} points that cover the space more evenly than \iid sampling, while ensuring each point is marginally uniform. 
We describe a number of different methods for doing this below.
This marginally uniform property is crucial. 
In particular, for the inverse CDF sampling in Section~\ref{sec:inverse-cdf} to be correct (\ie, each completion is an exact draw from the LM), the samples must be marginally uniform.
For a visualization of these methods, see Figure~\ref{fig:sampler_intuition}.

\label{sec:primer}

\paragraph{Lattice}
This procedure involves placing $k$ points on the regular grid $\{0, 1/k, \ldots, (k-1)/k\}$, drawing a single shared offset $\Delta \sim \mathrm{Unif}[0,1]$, giving
\[
    U_i \;=\; \Bigl(\tfrac{i-1}{k} + \Delta\Bigr) \bmod 1, \qquad i = 1, \ldots, k.
\]
This operation wraps points randomly along the unit interval.
Even though the $\{U_i\}_{i=1}^k$ are more spread-out by construction, the marginal distribution of each $U_i$ is exactly $\mathrm{Unif}[0,1]$ due to the random shift.
To see this, it is useful to think of this procedure as producing points on the unit circle: defining the endpoints as $0$ and $1$ turns $[0,1]$ into a circle with circumference one, and the $\bmod\,1$ operation is exactly a rotation that wraps points around it. 
The $k$ grid points sit at evenly spaced positions on the circle, and $\Delta$ rotates the entire configuration by a uniformly random angle. 
Because $\Delta$ is a uniform random shift, each grid point $i/k$ — viewed individually — is rotated to a uniformly random location on the circle, so $U_i \sim \mathrm{Unif}[0,1]$ marginally. 

\paragraph{Stratified} Instead of using a fixed-width grid, we can also divide the interval $[0,1]$ into $k$ equally sized strata and draw one point per stratum:
\[ U_i\sim\mathrm{Unif}\left[\tfrac{i-1}k,\tfrac ik\right]\qquad i=1,\dots,k. \]
This stratification ensures that no two samples can occupy the same subinterval. We then permute the resulting samples $\{U_i\}_{i=1}^k$ so that each $U_i$ is marginally distributed according to $\mathrm{Unif}[0,1]$.

The aforementioned QMC methods operate over the 1D space $[0,1]$. In the next section, we describe a procedure for mapping these 1D points to sequences using arithmetic coding.
Before doing so, we also present a third QMC method that operates in a higher dimensional space, bypassing the need for this procedure.

\paragraph{Token-level Sobol} Since we care about sampling token sequences, a natural approach is to incorporate the sequence dimension directly into the QMC space itself. Concretely, for sequence length $n$, we can produce uniform random variables $U\in[0,1]^n$ by using Sobol sequences~\citep{SOBOL196786}. Under this construction, the $j$-th coordinate of the point $U$ corresponds to the $j$-th token position, so we can decode by simply passing $U_j$ through the inverse CDF (Figure~\ref{fig:inverse_cdf}) and iterating. Points that collectively cover $[0,1]^n$ should thus produce LLM samples that collectively cover the model's distribution over sequences. We refer readers to e.g. PyTorch's open-source \texttt{SobolEngine} implementation for the actual technical details of constructing $\{U_i\}_{i=1}^k$.

We provide visual and intuitive comparisons for these three QMC methods in Section~\ref{sec:tradeoffs} alongside empirical results in Section~\ref{sec:pass_at_k}.

\subsection{From uniform random variables to token sequences via arithmetic coding}
\label{sec:inverse-cdf}
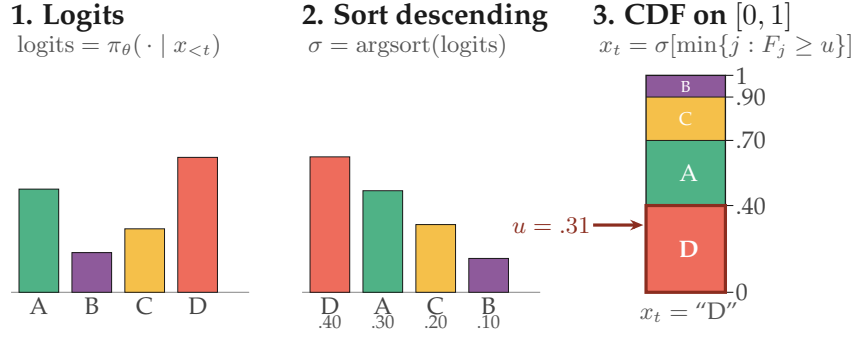
\begin{figure}[t]
    \centering

\providecolor{cGray}   {HTML}{8FA3C2}  \providecolor{cGrayD}   {HTML}{2C3E5C}
\providecolor{cGreen}  {HTML}{4FB286}  \providecolor{cGreenD}  {HTML}{1F5A40}
\providecolor{cTerra}  {HTML}{EE6C5C}  \providecolor{cTerraD}  {HTML}{8B2D1F}
\providecolor{cAmber}  {HTML}{E8A33D}
\providecolor{cPlum}   {HTML}{8E5A9B}
\providecolor{cInk}    {HTML}{1A1A1A}
\providecolor{cMid}    {HTML}{555555}
\providecolor{cAxis}   {HTML}{444444}

\begin{tikzpicture}[
  x=0.7cm, y=0.7cm,
  every node/.style={inner sep=0pt, outer sep=0pt},
  axis/.style={cInk, line width=0.4pt, opacity=0.5},
]

% =========================================================================
%  Question box — width capped to match the three-panel diagram (~10.5cm)
% =========================================================================
\node[draw=cAxis, line width=0.6pt, rounded corners=4pt,
      inner sep=8pt, anchor=north west, text width=10.2cm]
  at (0, 9.6) {%
    \centering
    {\small\bfseries Find all zeros in the indicated finite field of the given polynomial with coefficients in that field. $x^5 + 3x^3 + x^2 + 2x$ \quad in \quad $\mathbb{Z}_5$}\\[4pt]
    \makebox[\linewidth][c]{%
      A.\ $0$ \qquad
      B.\ $1$ \qquad
      C.\ $0,1$ \qquad
      \fbox{\textbf{D.\ $0,4$}}%
    }
  };

% =========================================================================
%  Panel headers
% =========================================================================
\node[cInk, font=\bfseries, anchor=west] at (0,    5.7) {1. Logits};
\node[cMid, font=\small,    anchor=west] at (0,    5.15) {$\mathrm{logits} = \pi_\theta(\,\cdot \mid x_{<t})$};

\node[cInk, font=\bfseries, anchor=west] at (5.5,  5.7) {2. Sort descending};
\node[cMid, font=\small,    anchor=west] at (5.5,  5.15) {$\sigma = \mathrm{argsort}(\mathrm{logits})$};

\node[cInk, font=\bfseries, anchor=west] at (11.0, 5.7) {3. CDF on $[0,1]$};
\node[cMid, font=\small,    anchor=west] at (11.0, 5.15) {$x_t = \sigma[\min\{j : F_j \geq u\}]$};

% =========================================================================
%  PANEL 1: raw logits as bars (A B C D, original token order)
% =========================================================================
\pgfmathsetmacro{\baseY}{0.5}
\draw[axis] (0, \baseY) -- (4.5, \baseY);

\foreach \i/\name/\hgt/\col in {%
    0/A/1.95/cGreen,%
    1/B/0.75/cPlum,%
    2/C/1.20/cAmber,%
    3/D/2.55/cTerra}%
{%
  \pgfmathsetmacro{\xL}{0.15 + \i*1.0}
  \pgfmathsetmacro{\xR}{\xL + 0.75}
  \fill[\col] (\xL, \baseY) rectangle (\xR, {\baseY+\hgt});
  \draw[cInk, line width=0.3pt] (\xL, \baseY) rectangle (\xR, {\baseY+\hgt});
  \node[cAxis, font=\footnotesize] at ({(\xL+\xR)/2}, {\baseY-0.25}) {\name};
}

% =========================================================================
%  PANEL 2: sorted-descending probability bars
% =========================================================================
\pgfmathsetmacro{\pscale}{6.4}
\begin{scope}[shift={(5.5, 0)}]
  \draw[axis] (0, \baseY) -- (4.5, \baseY);
  \foreach \i/\name/\prob/\plabel/\col in {%
    0/D/0.40/{.40}/cTerra,%
    1/A/0.30/{.30}/cGreen,%
    2/C/0.20/{.20}/cAmber,%
    3/B/0.10/{.10}/cPlum}%
  {%
    \pgfmathsetmacro{\xL}{0.15 + \i*1.0}
    \pgfmathsetmacro{\xR}{\xL + 0.75}
    \pgfmathsetmacro{\hgt}{\prob*\pscale}
    \fill[\col] (\xL, \baseY) rectangle (\xR, {\baseY+\hgt});
    \draw[cInk, line width=0.3pt] (\xL, \baseY) rectangle (\xR, {\baseY+\hgt});
    \node[cAxis, font=\footnotesize] at ({(\xL+\xR)/2}, {\baseY-0.25}) {\name};
    \node[cMid,  font=\scriptsize]   at ({(\xL+\xR)/2}, {\baseY-0.55}) {\plabel};
  }
\end{scope}

% =========================================================================
%  PANEL 3: CDF as vertical stacked bar on [0,1]
% =========================================================================
\pgfmathsetmacro{\Hcdf}{4.1}
\pgfmathsetmacro{\xCDFL}{12.0}
\pgfmathsetmacro{\xCDFR}{13.5}

\begin{scope}[shift={(0, 0)}]
  \pgfmathsetmacro{\yA}{\baseY}
  \pgfmathsetmacro{\yB}{\baseY + 0.40*\Hcdf}
  \pgfmathsetmacro{\yC}{\baseY + 0.70*\Hcdf}
  \pgfmathsetmacro{\yD}{\baseY + 0.90*\Hcdf}
  \pgfmathsetmacro{\yF}{\baseY + 1.00*\Hcdf}

  \fill[cTerra] (\xCDFL, \yA) rectangle (\xCDFR, \yB);
  \fill[cGreen] (\xCDFL, \yB) rectangle (\xCDFR, \yC);
  \fill[cAmber] (\xCDFL, \yC) rectangle (\xCDFR, \yD);
  \fill[cPlum]  (\xCDFL, \yD) rectangle (\xCDFR, \yF);

  \node[white, font=\bfseries\small] at ({(\xCDFL+\xCDFR)/2}, {(\yA+\yB)/2}) {D};
  \node[white, font=\small]          at ({(\xCDFL+\xCDFR)/2}, {(\yB+\yC)/2}) {A};
  \node[white, font=\scriptsize]     at ({(\xCDFL+\xCDFR)/2}, {(\yC+\yD)/2}) {C};
  \node[white, font=\tiny]           at ({(\xCDFL+\xCDFR)/2}, {(\yD+\yF)/2}) {B};

  \draw[cInk, line width=0.5pt] (\xCDFL, \yA) rectangle (\xCDFR, \yF);
  \foreach \y in {\yB, \yC, \yD} {
    \draw[cInk, line width=0.3pt] (\xCDFL, \y) -- (\xCDFR, \y);
  }

  \pgfmathsetmacro{\xtick}{\xCDFR + 0.15}
  \foreach \y/\lab in {\yA/0, \yB/{.40}, \yC/{.70}, \yD/{.90}, \yF/1} {
    \draw[cInk, line width=0.4pt] (\xCDFR, \y) -- ({\xtick}, \y);
    \node[cAxis, font=\footnotesize, anchor=west] at ({\xtick + 0.05}, \y) {\lab};
  }

  \pgfmathsetmacro{\yU}{\baseY + 0.31*\Hcdf}
  \draw[cTerraD, line width=1.2pt, -{Stealth[length=5pt,width=4pt]}]
    ({\xCDFL - 1.0}, \yU) -- ({\xCDFL - 0.05}, \yU);
  \node[cTerraD, font=\bfseries\small, anchor=east] at ({\xCDFL - 1.05}, \yU) {$u = .31$};

  \draw[cTerraD, line width=1.2pt] (\xCDFL, \yA) rectangle (\xCDFR, \yB);

  \node[cMid, font=\small, anchor=north] at ({(\xCDFL+\xCDFR)/2}, {\baseY - 0.2})
    {$x_t = \text{``D''}$};
\end{scope}

\end{tikzpicture}

\caption{\textbf{Exact sampling from the LM using inverse-CDF sampling.}
The LM produces logits over the answer choices.
We sort them in descending order to yield a permutation $\sigma$ and a sorted
probability vector~(panel~2).
Stacking these probabilities partitions the unit interval into bins whose widths
equal the token probabilities; a uniformly random point lands in each bin with
probability equal to that token's mass.
To sample a token, we first sample $u \in [0,1)$ and find the token whose
interval contains it.
$u = .31$ falls in $[0, .40)$, returning $x_t = \text{``D''}$~(panel~3).
For illustrative purposes, we show decoding a single token but repeating this process yields an entire sequence.
This recovers the same distribution as standard autoregressive sampling, which usually uses Gumbel-max.}
\label{fig:inverse_cdf}
\end{figure}
\label{sec:construction}
Randomized QMC gives us a method of sampling points in the unit interval that are (marginally) uniformly distributed but more evenly spread than \iid samples.
In particular, lattice and stratified sampling give us 1D samples $U\in[0,1]$. 
To map these into token sequences distributed according to the LM's autoregressive distribution $\pi_{\theta}$, we use arithmetic coding.
The key property of arithmetic coding is that it represents token sequences as sub-intervals on the unit interval whose \emph{lengths are equal to the sequence probabilities}. 
So, provided each point is uniformly sampled, it lands in a sub-interval with probability exactly equal to the probability the LM assigns to generating that sequence.
This property enables us to decode a set of QMC points into token sequences that are distributed according to $\pi_\theta$.

More formally, we construct a measure-preserving map from the unit interval to strings $\Phi:[0,1]\to\mathcal{V}^*$: a map with the property that if $U\sim\mathrm{Unif}[0,1]$, $\Phi(U)$ is distributed exactly according to $\pi_\theta$. 
By pushing marginally uniform QMC samples through $\Phi$, we can therefore get $k$ \emph{exact} samples from the language model.
Below, we describe this in more detail.

\paragraph{Arithmetic coding.}
We construct a map based on \emph{arithmetic coding}~\citep{cover2006elements,vilnis2023,mackay2002},
a classic idea from coding theory; we adapt the construction from~\citet{vilnis2023}.
Arithmetic coding represents sequences as nested
partitions of the unit interval. 
We describe how sub-intervals are mapped to sequences via a recursive procedure.
We start with $[0,1]$ and partition it into
bins whose widths are the first-token probabilities $\pi_\theta(x_1)$,
for $x_1\in\mathcal V$. If $u$ falls in the bin for $x_1$, then $x_1$ is the
first token. We then recurse inside that bin: the bin for $x_1$ is partitioned
into sub-bins whose relative widths are given by
$\pi_\theta(\cdot\mid x_1)$, so that the sub-bin for $x_2$ has width $\pi_\theta(x_1)\,\pi_\theta(x_2\mid x_1)$.

Continuing this process, each prefix $x_{1:t}$ is assigned an interval
$I(x_{1:t})\subseteq[0,1]$ of length
$|I(x_{1:t})| =\pi_\theta(x_{1:t}).$
These intervals form a prefix tree: the interval for a prefix contains
sub-intervals corresponding to all possible one-token continuations. 
Sampling a full sequence amounts to drawing a single $u\sim\mathrm{Unif}[0,1]$ and traversing the tree to a leaf node, by following
the unique chain of nested intervals containing $u$; see Figure~\ref{fig:method}.
Since the interval for each sequence has length equal to its probability under $\pi_\theta$, the
resulting sequence is distributed exactly according to the language model, provided the samples are marginally uniform.

\paragraph{Generating a sequence via descending the tree: a numerically stable implementation.}
Evaluating $\Phi(u)$ means finding the leaf interval containing $u$.
To do this, we descend the tree one token at a time, selecting a child via inverse CDF sampling using the language model next-token distribution.
A naive approach tracks the prefix intervals at each step $I(x_{1:t})$ and checks which sub-bucket $u$ belongs to; this is not numerically stable, since this involves representing the raw probability
$\pi_\theta(x_{1:t})$.
Conveniently, we can avoid representing the interval lengths explicitly, by carrying only the position of $u$
\emph{relative} to the current interval. 
Let $u_t\in[0,1]$ be that coordinate after
prefix $x_{<t}$, with $u_1=u$. At each step the conditional $\pi_\theta(\cdot\mid x_{<t})$
partitions $[0,1]$ into one bucket per token; we select the bucket containing $u_t$ by an
inverse-CDF lookup~(Figure~\ref{fig:inverse_cdf}),
\[
  x_t = F_t^{-1}(u_t\mid x_{<t}),\qquad
  F_t(v\mid x_{<t}) = \sum_{v'\le v}\pi_\theta(v'\mid x_{<t}),
\]
We then rescale that bucket $[\ell_t,\ell_t+p_t)$, with $p_t=\pi_\theta(x_t\mid x_{<t})$, back
to the unit interval,
\begin{equation}
  u_{t+1} = \frac{u_t-\ell_t}{p_t},
\label{ref:rescaled_us}    
\end{equation}
giving the local coordinate for the next conditional. 
Every step is then a categorical
draw followed by an affine rescaling. 
The selection reads the same logits already used for softmax decoding, so
the procedure folds into a standard autoregressive sampler at negligible cost; only the running coordinate $u_t$ is added as state. 
Because this coordinate is the sole per-sample state, the $k$ rollouts can run as $k$ independent autoregressive decodes that never interact; generation is as parallelizable as ordinary \iid sampling.

\subsection{Estimators that depend on the      joint}
Section~\ref{sec:setup} showed that for average-type estimators, dependent samples are drop-in replacements for i.i.d.\ samples. 
But not every quantity computed from a batch can be expressed in this way and some commonly used estimators do rely on independence.
This section discusses how to modify these estimators to account for dependent samples. 
More broadly, these case studies point to general recipes for retrofitting estimators built on independence so they remain valid — and unbiased — under dependent samples.

\subsubsection{Dyadic bootstrap estimator of pass@$k$.}
\label{sec:dyadic}
\definecolor{samplecol}{RGB}{30,110,180}
\definecolor{accentcol}{RGB}{220,90,40}
\definecolor{mutedcol}{RGB}{220,220,220}
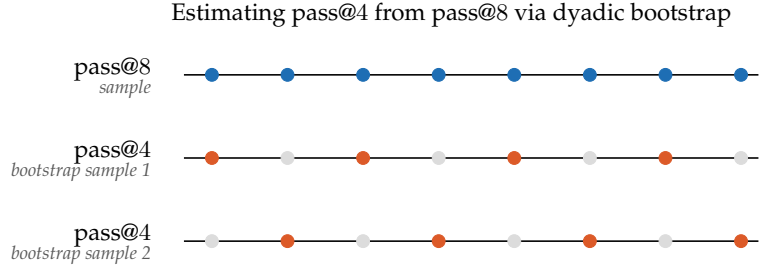
\begin{figure}[t]
  \centering
  \begin{tikzpicture}[
    bluept/.style={circle, fill=samplecol, draw=samplecol, inner sep=1.7pt},
    orangept/.style={circle, fill=accentcol, draw=accentcol, inner sep=1.7pt},
    mutedpt/.style={circle, fill=mutedcol, draw=mutedcol, inner sep=1.7pt},
    rowlabel/.style={anchor=east, font=\small},
    sublabel/.style={anchor=east, font=\scriptsize\itshape, text=black!65},
  ]
  \def\W{7}
  \pgfmathsetmacro{\R}{0.37}
  \def\Lend{7.6}
  \node[anchor=south, font=\small] at (\W/2, 0.5) {Estimating pass@$4$ from pass@$8$ via dyadic bootstrap};
  % pass@8 (full QMC lattice)
  \begin{scope}[yshift=0cm]
    \draw[line width=0.6pt] (0,0) -- (\Lend,0);
    \foreach \i in {0,1,...,7} {
      \pgfmathsetmacro{\x}{\i+\R}
      \node[bluept] at (\x,0) {};
    }
    \node[rowlabel] at (-0.3, 0.08) {pass@$8$};
    \node[sublabel] at (-0.3,-0.18) {sample};
  \end{scope}
  % pass@4, offset 0
  \begin{scope}[yshift=-1.1cm]
    \draw[line width=0.6pt] (0,0) -- (\Lend,0);
    \foreach \i in {1,3,5,7} {
      \pgfmathsetmacro{\x}{\i+\R}
      \node[mutedpt] at (\x,0) {};
    }
    \foreach \i in {0,2,4,6} {
      \pgfmathsetmacro{\x}{\i+\R}
      \node[orangept] at (\x,0) {};
    }
    \node[rowlabel] at (-0.3, 0.08) {pass@$4$};
    \node[sublabel] at (-0.3,-0.18) {bootstrap sample 1};
  \end{scope}
  % pass@4, offset 1
  \begin{scope}[yshift=-2.2cm]
    \draw[line width=0.6pt] (0,0) -- (\Lend,0);
    \foreach \i in {0,2,4,6} {
      \pgfmathsetmacro{\x}{\i+\R}
      \node[mutedpt] at (\x,0) {};
    }
    \foreach \i in {1,3,5,7} {
      \pgfmathsetmacro{\x}{\i+\R}
      \node[orangept] at (\x,0) {};
    }
    \node[rowlabel] at (-0.3, 0.08) {pass@$4$};
    \node[sublabel] at (-0.3,-0.18) {bootstrap sample 2};
  \end{scope}
  \end{tikzpicture}
  \caption{\textbf{Bootstrapped QMC pass@$k$ estimator.} Because the $k$ samples are dependent rather than i.i.d.\ draws, the standard pass@$k$ estimator is invalid; instead, we exploit that any stride-$2$ subsample is itself a valid pass@$k/2$ lattice. Stride $2$ admits two starting offsets, each yielding an unbiased pass@$4$ estimate from the same pass@$8$ rollout. This generalizes to larger strides.
  }
  \label{fig:bootstrap-passk}
\end{figure}
The typical way of estimating pass@$k$ involves an unbiased estimator introduced in~\citep{chen2021evaluatinglargelanguagemodels}: the procedure draws $N \geq k$ samples, counts the number $c$ that are correct, and computes the closed form
\[
    \widehat{\text{pass@}k} = 1 - \frac{\binom{N-c}{k}}{\binom{N}{k}},
\]
the fraction of size-$k$ subsets that contain at least one correct sample.
This avoids having to draw repeated size-$k$ batches.
This closed form solution is valid when samples are independent and cannot be applied directly to our \qmt samplers. 
For lattice sampling, we develop an analogous \emph{dyadic bootstrap estimator} that subsamples a batch of $N$ responses in a way that exploits the lattice structure to produce unbiased estimates.

\begin{theorem}
\label{thm:dyadic}
Correctness of dyadic bootstrap sampling.
Let \(k=2^L\), and let
\[
    u_i = \left(\Delta + \frac{i}{k}\right) \bmod 1,
    \qquad i=0,\ldots,k-1,
\]
where \(\Delta \sim \mathrm{Unif}[0,1]\). For any \(x \le L\), set \(m=k/2^x\). Then each stride-\(2^x\) subsequence
\[
    u^{(r)}_j := u_{r + 2^x j},
    \qquad j=0,\ldots,m-1,
\]
with \(r \in \{0,\ldots,2^x-1\}\), is distributed exactly as a fresh randomly shifted lattice of size \(m\). Consequently, each such subsequence gives an unbiased pass@\(m\) estimate under the randomized-lattice construction.
\end{theorem}
\begin{proof}[Proof sketch]
The full proof is given in Appendix~\ref{sec:theory_app} and the intuition is captured in Figure~\ref{fig:bootstrap-passk}; the key observation is that the stride-$2^x$ subsequence is a (random) lattice with spacing $1/m$, plus a deterministic shift. Because a randomly-shifted lattice is invariant under deterministic shifts modulo $1$, the subsequence is distributed as a randomly-shifted $m$-point lattice. 
\end{proof}

By similarly exploiting the structure of the stratified and Sobol samplers, we can also derive bootstrap pass@$k$ estimators for these methods, which we present in Appendix~\ref{sec:theory_app}.

\subsubsection{RLOO estimators.}
\definecolor{colI}{HTML}{C44E52} % K_i
\definecolor{colJ}{HTML}{4C72B0} % K_j
\definecolor{colO}{HTML}{55A868} % overlap

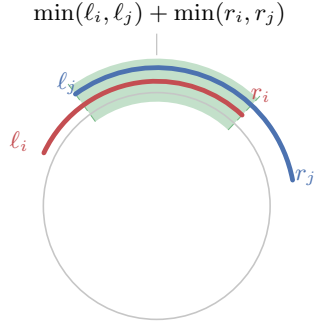
\begin{figure}[t]
  \centering
  \begin{tikzpicture}[scale=1.5,font=\footnotesize]
    % config (lengths as fractions of the circle):
    %   l_i > l_j  ->  min(l)=l_j   (nearer endpoint on the left)
    %   r_i < r_j  ->  min(r)=r_i   (nearer endpoint on the right)
    \def\lI{0.18}\def\rI{0.12}\def\lJ{0.10}\def\rJ{0.22}
    \def\mL{0.10}\def\mR{0.12} % min(l_i,l_j) , min(r_i,r_j)
    % angle(f) = 90 - 360 f : 0 at top, r to the right, l to the left

    % --- shaded overlap (annular sector across 0) ---
    \fill[colO,opacity=0.35]
      ({90-360*\mR}:0.92) arc[start angle={90-360*\mR},end angle={90+360*\mL},radius=0.92]
      -- ({90+360*\mL}:1.30) arc[start angle={90+360*\mL},end angle={90-360*\mR},radius=1.30]
      -- cycle;
    \draw[colO!80,line width=0.4pt,densely dashed] ({90+360*\mL}:0.92)--({90+360*\mL}:1.33);
    \draw[colO!80,line width=0.4pt,densely dashed] ({90-360*\mR}:0.92)--({90-360*\mR}:1.33);

    % --- torus outline ---
    \draw[gray!45,line width=0.6pt] (0,0) circle (1.0);

    % --- arcs (drawn over the top, outside the outline) ---
    \draw[colI,line width=1.8pt,line cap=round]
      ({90+360*\lI}:1.10) arc[start angle={90+360*\lI},end angle={90-360*\rI},radius=1.10];
    \draw[colJ,line width=1.8pt,line cap=round]
      ({90+360*\lJ}:1.22) arc[start angle={90+360*\lJ},end angle={90-360*\rJ},radius=1.22];

    % --- endpoint labels ---
    \node[colI] at ({90+360*\lI}:1.34) {$\ell_i$};
    \node[colJ] at ({90+360*\lJ}:1.34) {$\ell_j$};
    \node[colI] at ({90-360*\rI}:1.34) {$r_i$};
    \node[colJ] at ({90-360*\rJ}:1.34) {$r_j$};

    % --- result, tied to the shaded arc ---
    \draw[gray!55,line width=0.4pt] (90:1.33)--(90:1.52);
    \node[anchor=south,font=\small] at (90:1.5)
      {$\min(\ell_i,\ell_j)+\min(r_i,r_j)$};
  \end{tikzpicture}
  \caption{%
The joint probability of two sequences is determined by the intersection of intervals whose endpoints correspond to the relative positions of each sampled uniform within its final arithmetic interval. Shown is the no-wrap case; there is also a wrap-around case (Fig.~\ref{fig:circular-overlap}).}
  \label{fig:overlap-main}
\end{figure}
\label{sec:dependent_samples_rloo}
The second estimator is the leave-one-out estimator used by RLOO. 
This approach is only unbiased when the baseline is independent of the samples, which is violated under dependent samples. 
Concretely, for each sample $y_i$, RLOO computes an advantage by subtracting the mean reward of the other group members from the raw reward $r(y_i)$,
\[
A_i = r(y_i) - \frac{1}{n-1}\sum_{j\ne i} r(y_j),
\]
and the resulting gradient $\frac1n\sum_i A_i\,s_\theta(y_i)$ is an unbiased
policy gradient where $s_\theta(y_i)$ is the score $\nabla_{\theta} \log \pi_\theta (y_i)$.
The unbiasedness of this estimator relies on independence in a subtle way: under i.i.d.\ draws the
baseline $\frac{1}{n-1}\sum_{j\ne i} r(y_j)$ is independent of $y_i$, and thus $s_i$, so it is mean-zero in expectation.
Under QMC, we do not have this independence and using dependent samples with RLOO introduces bias via the baseline.
Empirically, we find this has a negligible impact, but for completeness we describe how we might correct for this theoretically.

\paragraph{A pairwise characterization of RLOO that admits an importance sampling correction.}
We could take an importance sampling approach to correct for this following~\citet{dimitriev2021carms}. 
The key idea is to rewrite RLOO as a product-of-differences (PoD) estimator.
Proposition~\ref{prop:rloo_as_prod} shows that RLOO is an average of pairwise PoD terms,
\begin{align}
g_{\mathrm{RLOO}}
=
\frac{1}{n(n-1)}\sum_{i\ne j}
\tfrac12\bigl(r(y_i)-r(y_j)\bigr)\bigl(s_\theta(y_i)-s_\theta(y_j)\bigr),
\label{eq:rloo_pod}
\end{align}
each of which is an expectation over the pair $(y_i,y_j)$ with respect to the independent product measure $\pi_\theta\otimes\pi_\theta$. 

However, under a coupled sampler, the pair is drawn from the joint law $q_{ij}$ instead. 
The PoD estimator makes it clear that an importance sampling correction involving pairs of sequences can yield unbiased estimates, if we can evaluate $q_{ij}$ and the importance ratio is defined.
Concretely, reweighting each term in Equation~\ref{eq:rloo_pod} by
\[
\rho_{ij}(y_i,y_j)=\frac{\pi_\theta(y_i)\pi_\theta(y_j)}{q_{ij}(y_i,y_j)}
\]
maps each term's expectation back to the independent product measure; by linearity of expectation, the reweighted
average is therefore an unbiased RLOO gradient. 
The entire correction thus reduces computing the probability that the coupled sampler produces a given pair of sequences.

For the lattice sampler, the joint law can be computed explicitly. 
As Proposition~\ref{prop:stable-circular-overlap} establishes, this can be computed by considering intersections between intervals whose endpoints depend on the relative position of a uniform sample within its final arithmetic interval.
We visualize this in Figure~\ref{fig:overlap-main}.
Unfortunately, the coupled law $q_{ij}$ does not have full support over pairs of
sequences: intuitively, this is because the lattice has only a single degree of
freedom, the shared shift. 
This violates the support condition needed for importance sampling: the target law must be absolutely continuous with respect to the proposal. 
When this holds, the importance ratio has expectation 1.
Empirically, we observed that the ratio was substantially below 1, consistent with support mismatch.

\section{Theoretical Analysis: an upper bound on pass@$k$ for samplers with a target marginal}
\label{sec:theory}
\begin{figure}[t]
\centering
   \includegraphics[width=0.4\linewidth]{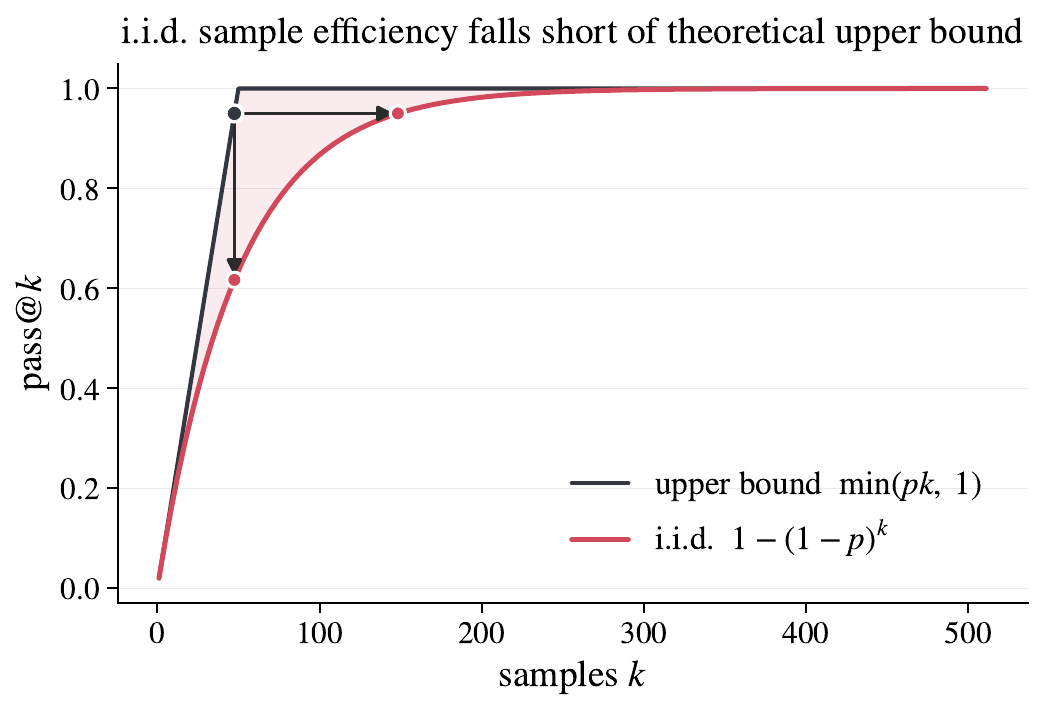}
\caption{\textbf{i.i.d. sampling falls short of the per-problem pass@$k$ upper bound.} For a problem with pass@1 probability $p$, the union bound gives a ceiling of $\min(pk, 1)$ (black) that places an upper bound on any sampler with target marginals.
The i.i.d. curve $1 - (1-p)^k$ (coral) lies strictly below this ceiling; the shaded gap represents the redundancy of independent sampling. 
Carefully constructed dependent samples can close the gap by being more spread out over the output space.}

\label{fig:pass_at_k_theory}
\end{figure}
Before evaluating \qmt empirically, we first ask how much room there is to improve beyond i.i.d sampling.
Intuitively, any sampler that respects the model's marginal distribution (\qmt included) can only rearrange the model's capability across $k$ rollouts; without training the model further, we cannot induce new capability. 
We formalize this as an upper bound on pass@$k$ in terms of pass@1, which gives us a ceiling against which to measure our gains empirically.

\begin{theorem}[Per-problem pass@$k$ upper bound]
\label{thm:pass-at-k-upper-bound}
Let $p$ denote the probability that a single sample solves
a fixed problem, i.e. its pass@1 probability. Let $P(k)$ denote the probability that at
least one of $k$ samples solves the problem. Then
\[
    P(k) \leq \min\{1, k p\}.
\]
\end{theorem}

\begin{proof}
Let $S_{j}$ be the event that the $j$-th sample solves the problem. 
By supposition, $\Pr(S_{j}) = p$ for each $j$. 
The event that the problem is solved by at
least one of the $k$ samples is given by $\bigcup_{j=1}^k S_{j}$ and by the union bound,
\[
    P(k)
    =
    \Pr\!\left(\bigcup_{j=1}^k S_{j}\right)
    \leq
    \sum_{j=1}^k \Pr(S_{j})
    =
    k p.
\]
We also must have $P(k) \leq 1$. Combining the
two bounds gives the desired result.
\end{proof}

As Figure~\ref{fig:pass_at_k_theory} shows, the gap between i.i.d. and the upper bound depends on the configuration: if the problem is too easy, then the model will sample a correct answer with probability 1 regardless of how it is sampled, and if the problem is too hard, both the i.i.d. and upper bound pass@$k$ vanish. The gap is maximized at $k=1/p$, i.e. when the problem is hard but feasible, where i.i.d. achieves a pass@$k$ of roughly $1-(1-\frac1k)^k\approx1-\frac1e=0.632$.

\paragraph{Tightness.} 
The bound is tight when the events $\{S_{j}\}_{j=1}^k$ are disjoint, i.e., when no two samples in the batch ever solve the problem in the same way simultaneously. 
Crucially, no sampling procedure that achieves this target pass@1 rate can do better than this upper bound, which places a theoretical limit on how good an entirely sampling-based, training-free method can be.

\section{Experiments}
We benchmark \qmt on four reasoning tasks chosen to understand the effect of the sampler: each is hard enough that coverage matters (i.i.d.\ does not trivially saturate), admits a verifiable reward, and uses a compact symbolic vocabulary, avoiding the semantic equivalence class issues that arise in certain natural language reasoning tasks.

\paragraph{Samplers.} 
We use \textbf{\qmt} to draw $k$ samples via the three approaches described in Section~\ref{sec:qmc}: lattice, stratified, and token-level Sobol.
We benchmark against
\textbf{i.i.d.} draws $k$ independent samples from the model and an \textbf{upper bound} that reflects the per-problem ceiling from
Theorem~\ref{thm:pass-at-k-upper-bound}. This characterizes the best pass@$k$
achievable by any sampler that preserves the LM's marginals. 
We estimate the per-problem pass@$k$ via the
bootstrap procedure of Section~\ref{sec:dyadic}. 
\paragraph{Tasks.}
In \textbf{Countdown} \citep{gandhi2024stream,gandhi2025cognitive,pan2025learning}, the model
is given a target number and a list of numbers and must combine them using
arithmetic operations ($+, -, \times, /$) to reach the target. 
In
\textbf{Maze} \citep{tajwar2026maximumlikelihoodreinforcementlearning}, the
model receives a 2D maze in text form and must output a sequence of moves
(Up/Down/Left/Right) from a start to an end position. In \textbf{Sudoku},
the model fills in a partially-completed $9 \times 9$ grid subject to the
standard constraints. In \textbf{1D-ARC}
\citep{xu2024llmsabstractionreasoningcorpus}, the model is shown a few
input/output pairs of one-dimensional grids and must infer and apply the
underlying transformation to a held-out input.
See the Appendix~\ref{sec:experiments_app} for representative
examples and solutions.

\paragraph{Models.}
For Countdown, Sudoku, and Maze, we use Qwen3.5-0.8B-Base. 
We observe that off-the-shelf LLMs perform poorly on these tasks: they struggle to adhere to format requirements (even instruction-tuned models) and hold a reasonable prior (e.g., they place zero mass on the correct solution). 
This makes it challenging to apply RL directly to these models~\citep{gandhi2025cognitive}.
Thus, for these three tasks, we perform supervised fine-tuning on 50,000 correct
question/answer pairs.
For 1D-ARC, we use Qwen3.5-2B-Base directly.
ARC inputs are inductive by nature --- containing multiple in-context examples --- so the base model already produces well-formatted outputs, and does not require any additional supervised fine-tuning. 

\paragraph{Evaluation.}
For each task we report pass@$k$ (\ie 1 if at least one of the $k$ attempts is correct),
and the pass@$k$ upper bound, averaged across all problems. 
To isolate the effect of the sampler, all
methods share the same model, prompt, and decoding temperature; only the joint
distribution of the underlying uniforms differs.
We use temperature $1.0$ and top-$p$ $0.95$
throughout, with a maximum completion length of $32$ tokens for Countdown and Maze, and $280$ tokens for Sudoku and 1D-ARC.

\subsection{Pass@$k$ results}
\label{sec:pass_at_k}
\paragraph{\qmt substantially outperforms i.i.d.\ and comes close to saturating the upper bound.}
In Figure~\ref{fig:pass_at_k},
we see that \qmt (lattice-based) sampler produces a significantly higher pass@$k$ than i.i.d.\ sampling at every value of $k$, across all four tasks.
We only present results for the lattice based sampler here, but we provide the pass@$k$ for all samplers in Figure~\ref{fig:pass_at_k_all_samplers} in the Appendix. 
Strikingly, the \qmt pass@$k$ curves track the upper bound closely; this means there is little room for \emph{any} training-free sampler (that preserves the marginals) to further improve performance.
This holds, to different degrees, across a diverse set of tasks with different solution geometries — arithmetic search (Countdown), spatial
planning (Maze), constraint satisfaction (Sudoku), and pattern induction (1D-ARC). 
Note that, although the upper bound from Section~\ref{sec:theory} is (piecewise) linear in $k$, it is a \emph{per-problem} upper bound and
the pass@$k$ curves in Figure~\ref{fig:pass_at_k} aggregate over problems with different accuracies.

\paragraph{\qmt sample efficiency.}
We can also ask how many fewer samples
\qmt needs to match a given i.i.d.\ pass rate; we present these results in Figure~\ref{fig:sample_efficiency}.
For the lattice-based sampler, \qmt reaches i.i.d.'s performance with a 25-47\% reduction in samples. 
Therefore, at deployment, we can use \qmt to reduce the rollout budget at matched quality or hold the budget fixed and spend the
freed FLOPs elsewhere.
The compute savings are essentially free since
the per-rollout overhead of \qmt is negligible (one extra uniform draw per batch
plus the inverse-CDF bookkeeping).
The other non-lattice QMC methods also improve sample-efficiency, although to a lesser extent.
\begin{figure}[t]
\centering
\includegraphics[width=1.0\linewidth]{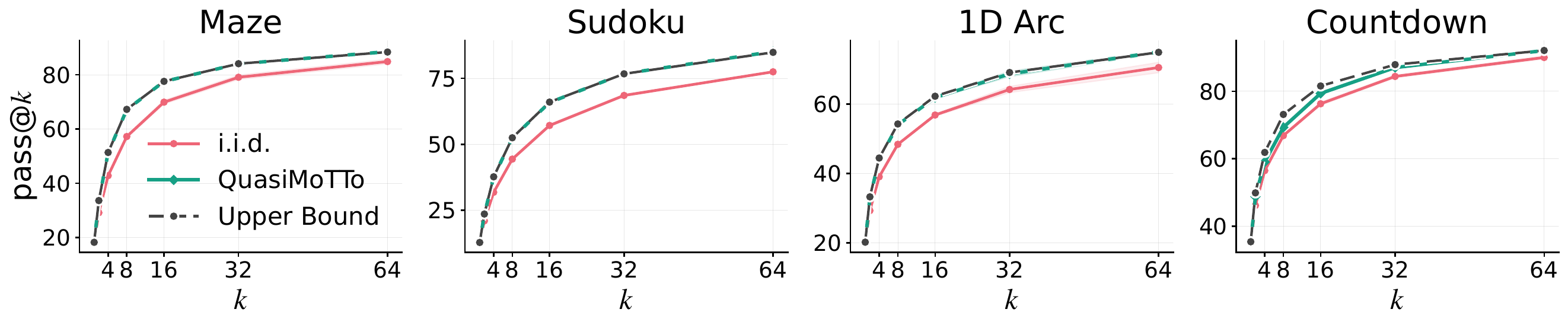}
\caption{\textbf{\qmt pass@$k$ analysis.}
\qmt (lattice) consistently
dominates i.i.d.\ sampling across all four reasoning benchmarks. 
Strikingly, \qmt closely tracks the pass@$k$ upper bound that no training-free sampler can exceed.
Note, we plot error bars $\pm 1$ SD but they are small enough that they are not visible in the plot above. 
Importantly, the upper bound curve aggregates over multiple problems, which is why it is not piecewise linear in $k$.
}
\label{fig:pass_at_k}
\end{figure}

\begin{figure}[t]
\centering
\includegraphics[width=\linewidth]{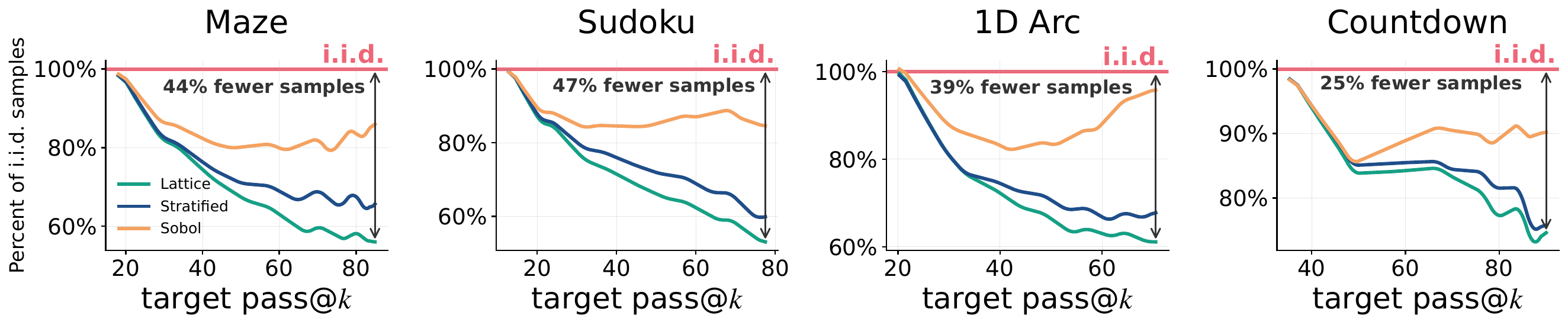}
\caption{\textbf{\qmt boosts sample efficiency.} \qmt can use 25\%-47\% fewer samples than i.i.d. sampling while achieving the same accuracy.
These performance gains hold across different QMC methods although lattice performs best.
}
\label{fig:sample_efficiency}
\end{figure}

\subsection{Policy-gradient based reinforcement learning}
\label{sec:rl-results}
\begin{figure}[t]
\centering
   \includegraphics[width=0.75\linewidth]{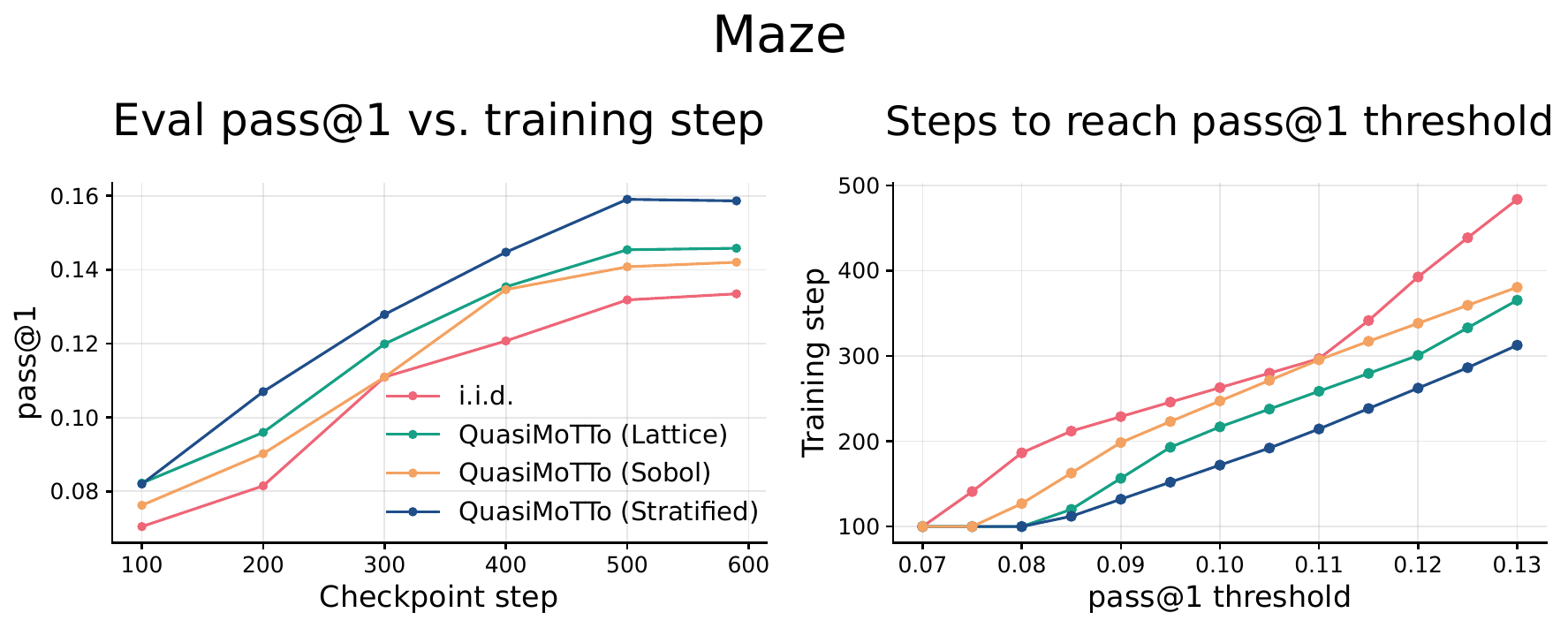}
   \includegraphics[width=0.75\linewidth]{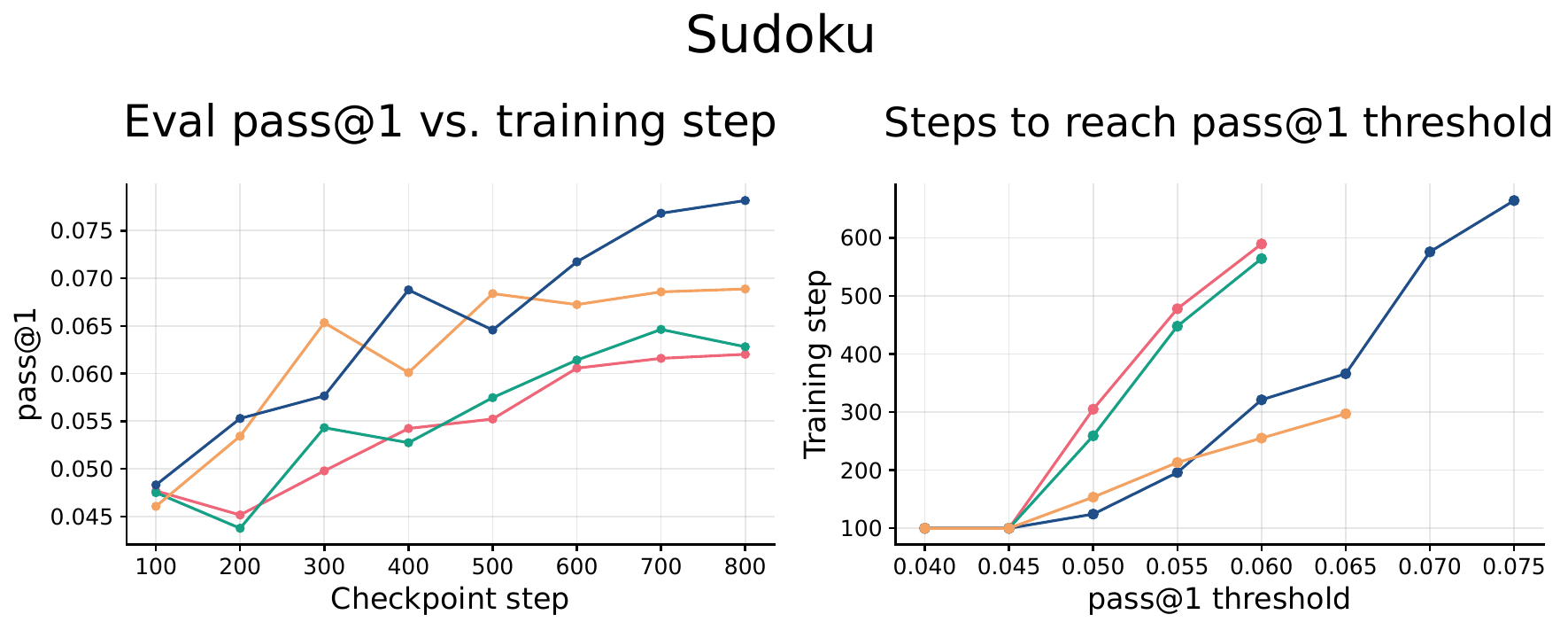}
\caption{\textbf{\qmt compute efficiency for RL}. 
We plot the pass@1 on evaluation set against the training step.
\qmt achieves the same pass@$1$ in fewer steps compared to i.i.d sampling.
}
\label{fig:rl_results}
\end{figure}
\begin{figure}[t]
\centering
   \includegraphics[width=0.65\linewidth]{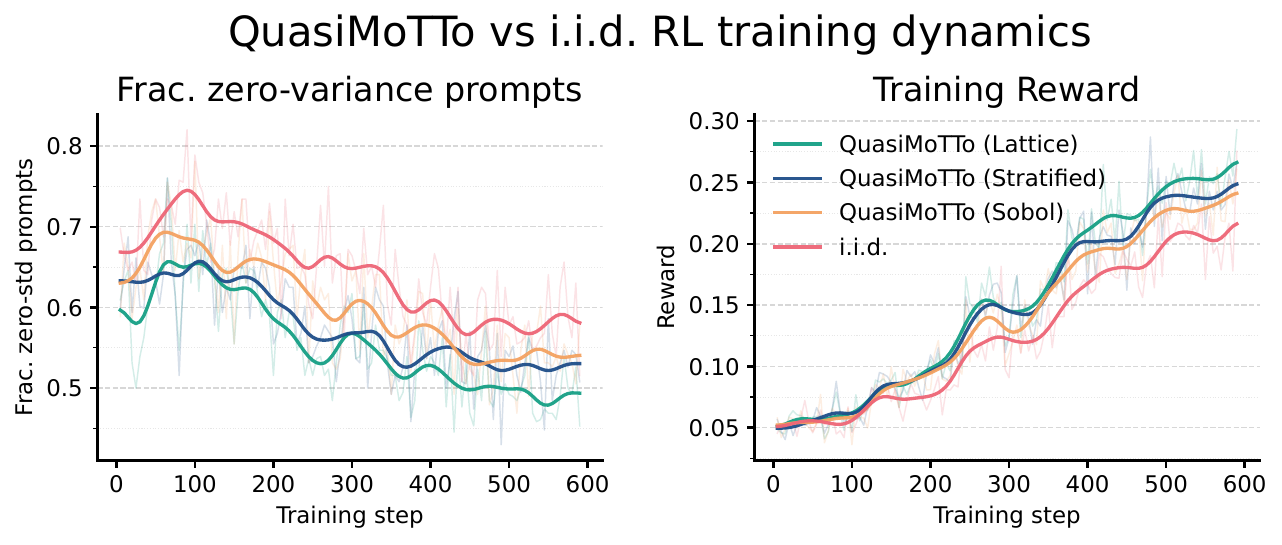}
\caption{\textbf{\qmt RL training dynamics}. 
\textbf{(left)}
QuasiMoTTo produces consistently fewer zero-variance groups (i.e., groups in which $G$
rollouts all succeed or all fail and contribute no learning signal)
throughout training because it samples responses with higher coverage.
\textbf{(right)}
This larger effective sample size causes \qmt's training reward to rise faster.
}
\label{fig:rl_training_dynamics}
\end{figure}

Since \qmt produces exact samples from the LM, we can use it as a \emph{drop-in replacement} for policy-gradient based reinforcement learning.
If \qmt can produce more diverse outcomes within a group, it can boost the effective sample size per gradient step and decrease the number of training steps required to achieve a target performance.

\paragraph{Setup.} We train Qwen3.5-0.8B-Base on Maze and Sudoku using GRPO~\citep{guo2025deepseekr1} to measure whether \qmt rollouts translate into faster policy improvement. 
As a baseline, we also train the model using an i.i.d.\ sampler.
We use a group size of $G = 8$, $64$ unique prompts per gradient step, sampling temperature
$\tau = 1.0$, and a max completion length of 256 tokens.
We perform 2 gradient steps per rollout batch (\ie, 2 off-policy updates).
We use sequence level off-policy importance sampling~\citep{zheng2025groupsequencepolicyoptimization}.
Note that, because we use a group baseline with dependent samples, we introduce some amount of bias; see Section~\ref{sec:dependent_samples_rloo}.

Experiments were run on 4 H100s/A100s.
We organize our analysis around two questions: (1) does using \qmt during RL
training translate into a compute-efficiency win (fewer steps to get the same reward), and (2) how do correlated samples change the training dynamics?
% \vspace{-1em}
\paragraph{GRPO background.}
We provide some basic background on GRPO.
Concretely, for a prompt $x$, the policy $\pi_{\theta}$ generates $G$ completions 
$y_1,\ldots,y_G$ and receives rewards
$r_1,\ldots,r_G$. 
GRPO forms a group-relative advantage for each completion by comparing its reward to the rewards of the other completions from the same prompt; the $i$-th completion gets advantage $A_i = r_i - \mathrm{mean}(r_1,\ldots,r_G)$. 
If all $G$ completions for a prompt receive the same reward, then the group has zero reward variance and contributes no policy-gradient signal. 
We call such a group a \emph{zero-variance group} when all rollouts are correct, or incorrect. 
Zero variance groups are a significant practical challenge: they increase the policy gradient variance, by reducing the number of prompts the LM gets a learning signal from in a batch (\ie, the effective sample size ~\citep{yu2025dapoopensourcellmreinforcement}).

\paragraph{Hard-problem pass@$k$ filtering.} In GRPO-style training, learning is primarily driven by rare but correct rollouts; questions that are too hard or easy result in more zero-variance groups, less learning signal, and increased noise~\citep{yu2025dapoopensourcellmreinforcement,goel2026speedinguprlhighleverage}. As established in Section~\ref{sec:theory}, we also expect \qmt samplers to exhibit the most improvement over i.i.d. in this hard-but-feasible regime. Thus, to better study the effects of \qmt on RL training, we filter out Maze and Sudoku problems which are too easy or too hard. We do this as a preprocessing step by generating 32 completions per problem with the base model checkpoint, and filtering out problems with an average pass rate $=0$ or $>\frac{1}{G}=0.125$.
 
\paragraph{\qmt is a compute efficiency win.}
In Figure~\ref{fig:rl_results}, we compare the evaluation set pass@$1$ for \qmt against i.i.d. sampling; for evaluation, we use \qmt sampler for both methods to isolate the benefits of using the QMC sampler at training time.
\qmt provides a consistent compute-efficiency win; we can achieve the same target pass@1 in 50\% fewer steps.
In the Appendix, we report error bars~(Figure~\ref{fig:rl_results_error_bars}).

\paragraph{\qmt sampling improves training dynamics by reducing zero-gradient groups.}
With i.i.d. sampling, the rollouts in a group frequently coincide, in which case the group-relative advantage collapses to zero and the
gradient update carries no signal. 
In Figure~\ref{fig:rl_training_dynamics}, we see that \qmt actively mitigates
this: \qmt encourages rollouts to occupy different regions of the sequence space.
We see this clearly reflected in a smaller percentage of zero-variance groups.
The result is faster improvement due to a larger effective sample size.
Interestingly, although dependent samples introduce some amount of bias through the group relative baseline (Section~\ref{sec:dependent_samples_rloo}), this is more than compensated by the increase in the effective batch size.

\subsection{Theoretical and empirical tradeoffs between samplers}
\label{sec:tradeoffs}
In this section we provide a high-level picture for how \qmt samplers (lattice, stratified, and Sobol) relate to each other and to i.i.d. sampling. The goal is to build intuition for why different samplers lead to the different pass@$k$ and RL results shown above.
\input{figure_tex/sampler_intuition}

\paragraph{Freedom vs. coverage.} Of the three 1D samplers, independent sampling has the most freedom: samples do not influence one other at all, and coverage guarantees are weak. As we move to stratified and then lattice sampling, the ``freedom" gets more restricted --- points repel one another more strongly --- and the coverage consequently increases (Figure~\ref{fig:sampler_intuition}). This explains why pass@k monotonically increases from i.i.d. to stratified to lattice in Section~\ref{sec:pass_at_k}, and why the proportion of zero-variance groups drops in Section~\ref{sec:rl-results}.

One way to quantify this tradeoff is through the \emph{pairwise mutual information} 
\[ I(U_i; U_j)=H(U_i)-H(U_i\mid U_j), \]
an information-theoretic quantity which measures how much information one gains about a random variable by conditioning on another variable. For independent samples this is trivially 0. For stratified sampling, conditioning on $U_j$ reduces the possible space of $U_i$ from $[0,1]$ to a region of length $\frac{k-1}k$, so the mutual information is
\[ -\int_0^1\log 1 \, \mathrm{d}u +\int_0^{\frac{k-1}k}\frac k{k-1}\log\frac k{k-1} \, \mathrm{d}u =\log\frac k{k-1}. \]
For lattice sampling, since conditioning on any point $U_j$ makes any other point $U_i$ deterministic, the distribution collapses to a point mass and the differential entropy $H(U_i\mid U_j)$ is $-\infty$.

We note that stratified empirically outperforms lattice when applied to RL (Section~\ref{sec:rl-results}), which is seemingly at odds with this freedom vs. coverage intuition. We hypothesize that this is a consequence of the RLOO bias mentioned in Section~\ref{sec:dependent_samples_rloo}: lattice has higher coverage but higher bias, which overall hurts training more than stratified.

\paragraph{Sequence-level vs. token-level coverage.} Unlike lattice/stratified sampling, Sobol sequences are defined over a high-dimensional space. The coverage applies locally at each token position rather than globally for the whole sequence. We would therefore expect Sobol to perform somewhat weaker than stratified or lattice, but still stronger than i.i.d., an intuition which holds in the pass@$k$ and RL experiments above.
\section{Related Work}
% \vspace{-1em}
\paragraph{Beam Search}
Beam search and its diverse variants~\citep{vijayakumar2018diverse} also boost diversity, but their outputs do not correspond to draws from a well-defined distribution, let alone the LM's, which precludes their use within procedures like RL (they would introduce substantial bias).
Stochastic beam
search~\citep{pmlr-v97-kool19a} reduces redundancy by drawing structured
sets of distinct sequences. 
However, because sampling without replacement
changes the sequence probabilities, unbiased estimation
with respect to the LM requires importance weight corrections.
This suffers from the conventional challenges of importance weighting such as high variance.

\paragraph{Exploration in RL}
Recent approaches modify the RL or SFT training objective to induce broader coverage over rollouts \citep{tajwar2026maximumlikelihoodreinforcementlearning,hamid2026polychromic,orney2026polyepotrainingexploratoryreasoning,chen2026rethinking, hamid2026spirallearningsearchaggregate}. 
Our approach is orthogonal: we intervene at the sampling stage rather than during training, leaving the objective untouched.

\paragraph{Efficient sampling}
Most work on efficient LLM inference lowers the cost of \emph{each} rollout.
IO-aware exact attention removes the memory bottleneck of softmax
attention~\citep{dao2022flashattention}, paged memory management raises serving
throughput~\citep{kwon2023vllm}, sparse attention drops the quadratic
scaling in context length~\citep{deepseekai2025deepseekv32pushingfrontieropen}, and learned KV cache compression methods reduce the size of the context~\citep{li2026neuralgarbagecollectionlearning, mao2026simplifiedsparseattentiongist}.
Orthogonally, one can reduce latency by restructuring a single chain of thought into subtasks executed in
parallel~\citep{mahankali2026divideandconquercotrlreducing,pan2025learning}. 
All of these make a rollout
cheaper or faster; \qmt instead reduces the \emph{number} of
rollouts needed to reach a target performance.

\section{Conclusions \& Limitations}
% \vspace{-1em}
We introduced QuasiMoTTo, an approach that combines randomized QMC and inverse CDF sampling to generate a dependent batch of LM rollouts for scaling inference compute.
Each rollout is an exact sample from the LM, but the batch as a whole covers the output space more evenly. This buys compute efficiency: 25–47\% fewer samples to match i.i.d. pass@$k$ on reasoning benchmarks, often saturating the union-bound ceiling that no marginal-preserving sampler can exceed; and 50\% fewer GRPO steps to hit a target pass@1, driven by a drop in zero-variance groups. 

Our experiments use 1–2B parameter models on tasks with shorter outputs for challenging symbolic reasoning tasks.
Extending QuasiMoTTo to long chain-of-thought reasoning requires defining a notion of coverage over semantic equivalence classes of solutions. 
It would also be interesting to see if \qmt can boost coverage and diversity in open-ended tasks where this is particularly important, like scientific discovery~\citep{heyueya2026giantsgenerativeinsightanticipation, li2024automatedstatisticalmodeldiscovery,novikov2025alphaevolvecodingagentscientific, gandhi2025boxinggymbenchmarkingprogressautomated}.
These are natural directions for future work.

\section{Acknowledgements}
We thank members of Cocolab and Dynamode Lab, particularly Ivy Zhang and Qizhong Zhang, for thoughtful feedback on this project.
We thank Omar Shaikh and Matt Jorke for insightful discussions.
We thank Jubayer Ibn Hamid and Suvir Mirchandi for helpful discussions.
Some of the experiments for this project was performed on the Marlowe cluster at Stanford University.
This work was supported in part by ONR Grant N00014-22-1-2110, NSF Grant 2205084, and the Stanford Institute for Human-Centered Artificial Intelligence (HAI). EBF is a Biohub, San Francisco, Investigator.
\newpage
%%%%%%%%%%%%%%%%%%%%%%%%%%%%%%%%%%%%%%%%%%%%%%%%%%%%%%%%%%%%

\bibliographystyle{plainnat}
\bibliography{references}

\newpage
\appendix

\section{Technical appendices and supplementary material}
\subsection{Experiment Details}
\label{sec:experiments_app}
\begin{figure}[ht]
\centering
\small
\begin{tabular}{@{}p{0.42\linewidth}p{0.52\linewidth}@{}}
\textbf{Maze} &
\textbf{Countdown} \\[2pt]
\begin{minipage}[t]{\linewidth}
\begin{verbatim}
*********
*E..*...*
*.***.***
*.......*
***.*.***
*...*...*
*S*.*****
*.*.....*
*********
\end{verbatim}
{\color{red}
\begin{verbatim}
URRUULLUU

\end{verbatim}
}
\end{minipage}
&
\begin{minipage}[t]{\linewidth}
\begin{verbatim}
Target: 89
Numbers: [57, 43, 11]
\end{verbatim}
{\color{red}
\begin{verbatim}
43+57-11
\end{verbatim}
}
\end{minipage}
\\[6pt]
\textbf{Sudoku} &
\textbf{1D-ARC} \\[2pt]
\begin{minipage}[t]{\linewidth}
\begin{verbatim}
8 4 5 2 _ _ 1 3 7
2 _ 9 _ 5 3 _ _ 6
6 _ _ _ _ _ _ _ 5
7 5 _ 9 8 6 _ 1 3
1 _ _ _ 4 _ _ _ 2
4 _ 6 _ _ _ _ _ 8
_ 2 _ _ 3 _ _ 8 9
_ 8 _ _ _ 1 6 2 _
9 6 _ 8 7 _ 3 5 1
\end{verbatim}
{\color{red}
\begin{verbatim}
8 4 5 2 6 9 1 3 7
2 1 9 7 5 3 8 4 6
6 7 3 4 1 8 2 9 5
7 5 2 9 8 6 4 1 3
1 9 8 3 4 7 5 6 2
4 3 6 1 2 5 9 7 8
5 2 1 6 3 4 7 8 9
3 8 7 5 9 1 6 2 4
9 6 4 8 7 2 3 5 1
\end{verbatim}
}
\end{minipage}
&
\begin{minipage}[t]{\linewidth}
\begin{verbatim}
Input: [0,0,5,0,0,0,0,0,0,0,0,5,0,0]
Output: [0,0,5,5,5,5,5,5,5,5,5,5,0,0]

Input: [0,0,7,0,0,0,0,0,0,0,0,7,0,0]
Output: [0,0,7,7,7,7,7,7,7,7,7,7,0,0]

Input: [2,0,0,0,0,0,0,0,0,2,0,0,0,0]
Output: [2,2,2,2,2,2,2,2,2,2,0,0,0,0]

Input: [7,0,0,0,7,0,0,0,0,0,0,0,0,0]
Output:
\end{verbatim}
{\color{red}
\begin{verbatim}
[7,7,7,7,7,0,0,0,0,0,0,0,0,0]
\end{verbatim}
}
\end{minipage}
\\
\end{tabular}
\caption{Representative examples from each of the four tasks.}
\label{fig:example_tasks}
\end{figure}

\begin{table}[h]
\centering
\caption{Hyperparameters for Maze and Sudoku experiments.}
\label{tab:hparams-maze-sudoku}
\begin{tabular}{@{}ll@{}}
\toprule
\textbf{Hyperparameter} & \textbf{Value} \\
\midrule
Optimizer                    & AdamW \\
Learning rate                & $1 \times 10^{-6}$ \\
$(\beta_1, \beta_2)$         & $(0.9,\ 0.95)$ \\
$\epsilon_{\text{Adam}}$     & $1 \times 10^{-8}$ \\
Weight decay                 & $0.01$ \\
Max gradient norm            & $1.0$ \\
KL coefficient ($\beta$)     & $0.01$ \\
PPO clip range ($\epsilon$)  & $0.2$ \\
Off-policy epochs            & $2$ \\
\bottomrule
\end{tabular}
\end{table}

\newpage
\subsection{Additional experimental results}
\begin{figure}[ht!]
\centering
\includegraphics[width=1.0\linewidth]{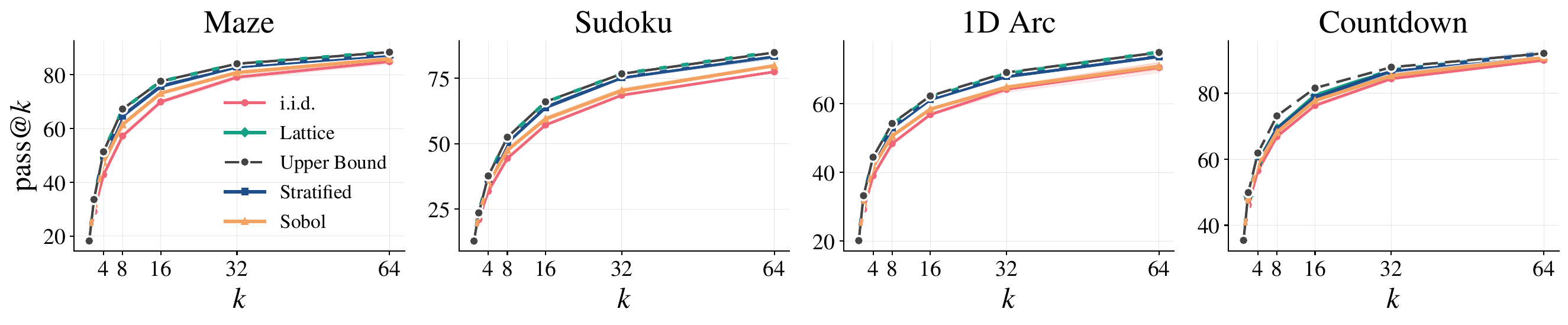}
\caption{\textbf{\qmt pass@$k$ analysis for all samplers.}
}
\label{fig:pass_at_k_all_samplers}
\end{figure}
\begin{figure}[t]
\centering
   \includegraphics[width=0.75\linewidth]{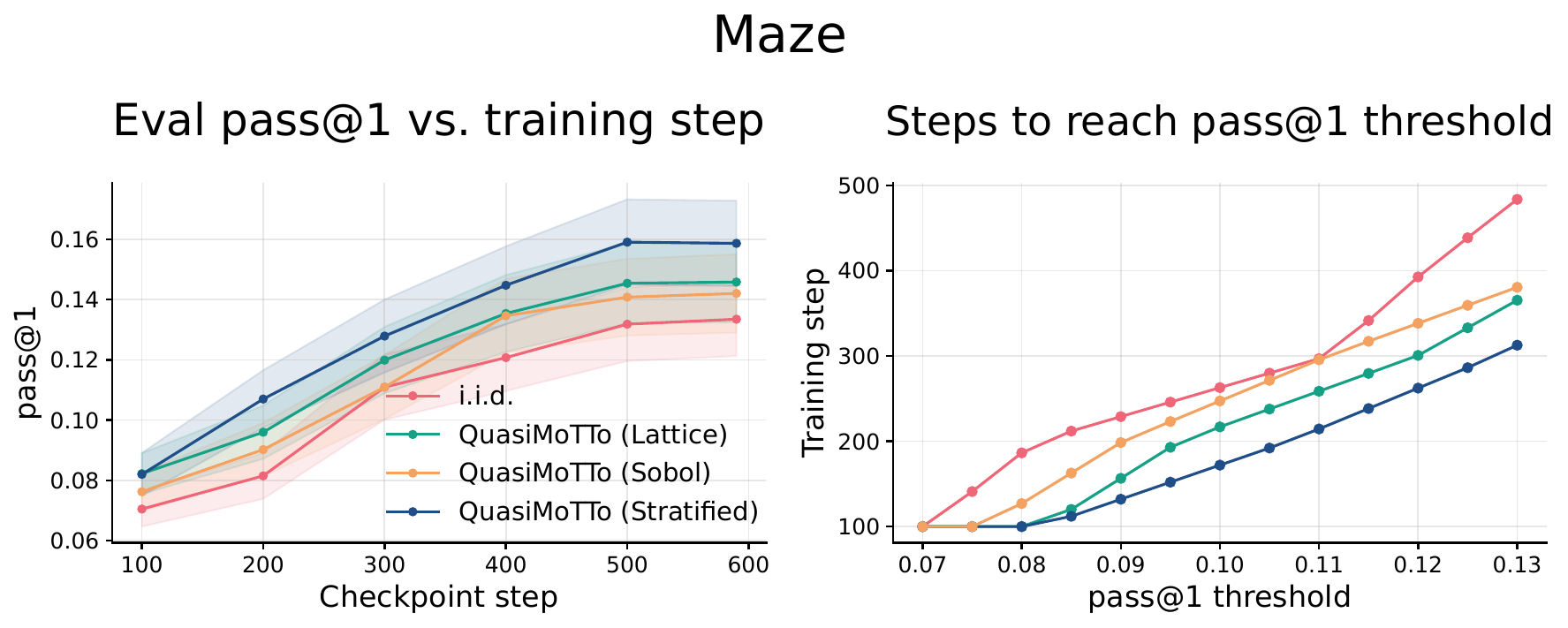}
   \includegraphics[width=0.75\linewidth]{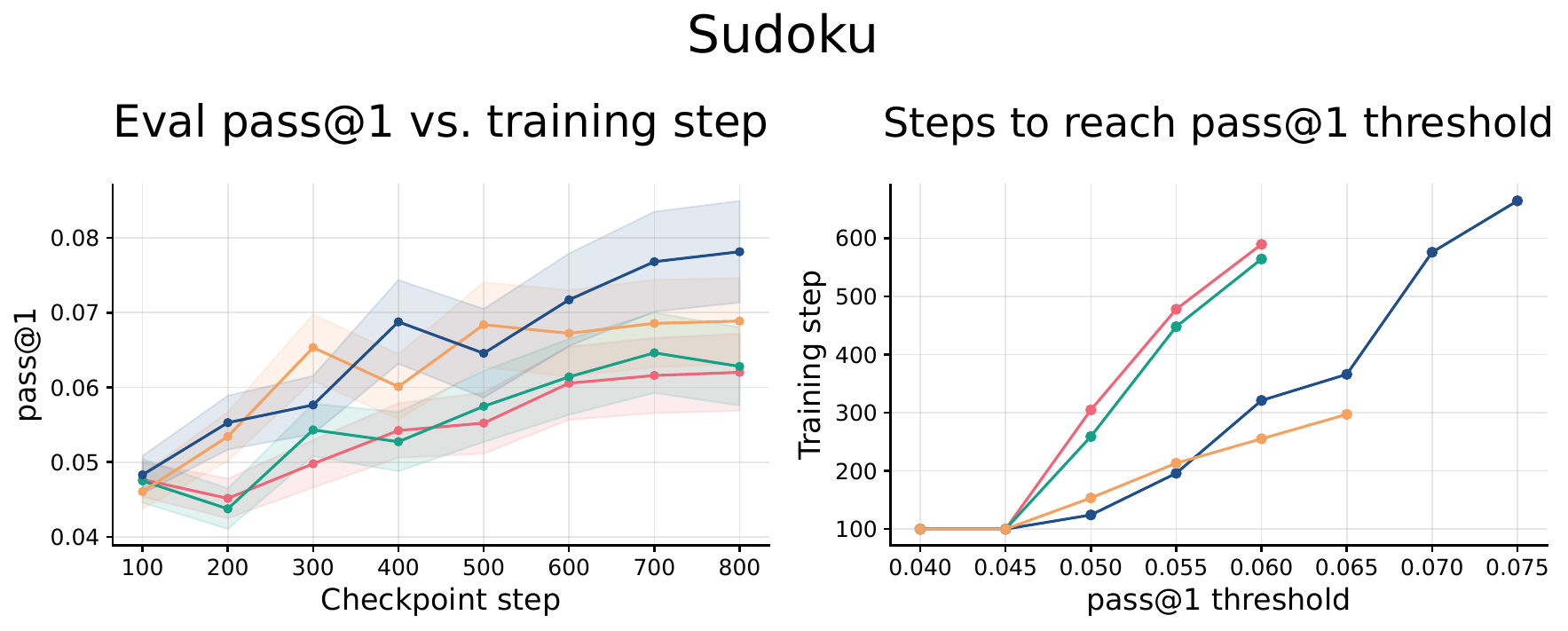}
\caption{\textbf{\qmt compute efficiency for RL}. 
We plot the pass@1 on evaluation set against the training step.
\qmt achieves the same pass@$1$ in fewer steps compared to i.i.d sampling.
Error bars are computed as follows: 
we estimate the accuracy $p_i$ for each of $n$ questions and report $\frac{\operatorname{std}(p)}{\sqrt{n}}$. 
}
\label{fig:rl_results_error_bars}
\end{figure}
\newpage
\subsection{Theoretical analyses}
\label{sec:theory_app}
\setcounter{theorem}{\numexpr\getrefnumber{thm:dyadic}-1}
\begin{theorem}[Correctness of dyadic bootstrap sampling]
Let \(k=2^L\), and let
\[
    u_i = \left(\Delta + \frac{i}{k}\right) \bmod 1,
    \qquad i=0,\ldots,k-1,
\]
where \(\Delta \sim \mathrm{Unif}([0,1])\). For any \(x \le L\), set \(m=k/2^x\). Then each stride-\(2^x\) subsequence
\[
    u^{(r)}_j := u_{r + 2^x j},
    \qquad j=0,\ldots,m-1,
\]
with \(r \in \{0,\ldots,2^x-1\}\), is distributed exactly as a fresh randomly shifted lattice of size \(m\). Consequently, each such subsequence gives an unbiased pass@\(m\) estimate under the randomized-lattice construction.
\label{ref:theorem_2}
\end{theorem}
\begin{proof}
Let \(s=2^x\), so \(m=k/s\). For any offset \(r\),
\[
    u_{r+sj}
    =
    \left(\Delta + \frac{r+sj}{k}\right)\bmod 1
    =
    \left(\Delta + \frac{r}{k} + \frac{j}{m}\right)\bmod 1.
\]
The last equality substitutes $\frac{1}{m}$ for $\frac{s}{k}$.
Define
\[
    \Delta^{(r)} := \left(\Delta + \frac{r}{k}\right)\bmod 1.
\]

For a fixed lattice subsequence, $r$ and $k$, are constants, so $\frac{r}{k}$ is a deterministic shift. 
Since adding a deterministic constant modulo 1 preserves the uniform distribution \(\Delta^{(r)}\) is also uniform on \([0,1]\). 
Finally substituting \(\Delta^{(r)}\)
and noting that adding a constant before reducing modulo \(1\) is the same as
first reducing the shifted quantity modulo \(1\), we obtain
\[
    u_{r+sj}
    =
    \left(\Delta^{(r)} + \frac{j}{m}\right)\bmod 1,
\]
which is exactly an \(m\)-point lattice with a uniform random shift.
Therefore any pass@\(m\) estimator computed on this subsequence has the same expectation as the estimator computed on a fresh \(m\)-point randomized lattice.
\end{proof}

\begin{theorem}[Bootstrap estimator for stratified sampling]
    Following the notation above, let $\{u_i\}_{i=0}^{k-1}$ be a set of stratified samples
    \[ u_i\sim\mathrm{Unif}[\tfrac ik, \tfrac{i+1}k],\qquad i=0,\ldots,k-1. \]
    Then, to obtain a bootstrapped stratified sample for $m=k/2^x$, we can partition $u_i$'s into contiguous blocks of size $k/m=2^x$
    \[ \{u_0,\ldots,u_{k/m-1}\}\cup\{u_{k/m},\ldots,u_{2k/m-1}\}\cup\cdots\cup\{u_{(m-1)k/m},\ldots,u_{k-1}\} \]
    and sample once from each partition
    \[ u_i'\sim\{u_{ik/m},\ldots,u_{(i+1)k/m-1}\} \]
    so that $\{u_i'\}_{i=0}^{m-1}$ is distributed according to the stratified rule for $m$ samples, i.e., $u_i\sim\mathrm{Unif}[\tfrac im, \tfrac{i+1}m]$.
\end{theorem}

\begin{theorem}[Bootstrap estimator for Sobol sampling]
    Let $\{u_i\}_{i=0}^{k-1}$ be a length $k$ Sobol sequence in dimension $n$
    \[ u_i\in[0,1]^n. \]
    Sobol sequences of length $2^L$ are constructed such that any prefix of length $2^x<2^L$ is itself a valid Sobol sequence. So, to obtain an unbiased pass@$m$ estimate, we can simply take contiguous blocks of $k/m$:
    \[ u_j^{(r)}:=u_{rm+j},\qquad j=0,\ldots,m-1, \]
    which gives bootstrapped samples for $r\in\{0,\ldots,2^x-1\}$.
\end{theorem}

The proofs for the above two bootstrap estimators follow analogously to the proof above for the dyadic estimator for lattice sampling.

\setcounter{proposition}{\numexpr\getrefnumber{prop:rloo_as_pod}}
\begin{proposition}[Characterization of RLOO as a product-of-differences estimator]
\label{prop:rloo_as_prod}

Let \(y_1,\dots,y_n\) be sampled trajectories from a policy
\(\pi_\theta(y)\). Let \(r(y_i)\in\mathbb R\) denote the reward of trajectory
\(y_i\), and define the score
\[
s_\theta(y_i) := \nabla_\theta \log \pi_\theta(y_i).
\]
The leave-one-out policy-gradient estimator
\[
g_{\mathrm{RLOO}}(\theta)
=
\frac{1}{n}
\sum_{i=1}^n
\left(
r(y_i)
-
\frac{1}{n-1}\sum_{j\ne i} r(y_j)
\right)
s_\theta(y_i)
\]
can be written equivalently as
\[
g_{\mathrm{RLOO}}(\theta)
=
\frac{1}{n(n-1)}
\sum_{i\ne j}
\frac{1}{2}
\left(r(y_i)-r(y_j)\right)
\left(s_\theta(y_i)-s_\theta(y_j)\right).
\]
\end{proposition}

\begin{proof}
Starting from the leave-one-out estimator,
\[
g_{\mathrm{RLOO}}(\theta)
=
\frac{1}{n}
\sum_{i=1}^n
\left(
r(y_i)
-
\frac{1}{n-1}\sum_{j\ne i} r(y_j)
\right)
s_\theta(y_i).
\]
Expanding,
\[
g_{\mathrm{RLOO}}(\theta)
=
\frac{1}{n}
\sum_{i=1}^n r(y_i)s_\theta(y_i)
-
\frac{1}{n(n-1)}
\sum_{i=1}^n
\sum_{j\ne i}
r(y_j)s_\theta(y_i).
\]

We first make the left and right term symmetric by introducing dummy indices that allow us to re-write the first term as a pairwise sum:
\[
\frac{1}{n}
\sum_{i=1}^n r(y_i)s_\theta(y_i)
=
\frac{1}{n}
\sum_{i=1}^n \frac{1}{n-1} \sum_{j \neq i} r(y_i)s_\theta(y_i)
=
\frac{1}{n(n-1)}
\sum_{i\ne j}
r(y_i)s_\theta(y_i).
\]
Therefore,
\[
g_{\mathrm{RLOO}}(\theta)
=
\frac{1}{n(n-1)}
\sum_{i\ne j}
\left(
r(y_i)s_\theta(y_i)
-
r(y_j)s_\theta(y_i)
\right).
\]

Since we can simply relabel $i$ with $j$, we also have
\[
\sum_{i\ne j}
\left(
r(y_i)s_\theta(y_i)
-
r(y_j)s_\theta(y_i)
\right)
=
\sum_{i\ne j}
\left(
r(y_j)s_\theta(y_j)
-
r(y_i)s_\theta(y_j)
\right).
\]
Therefore, we can write the RLOO estimator as
\[
g_{\mathrm{RLOO}}(\theta)
=
\frac{1}{n(n-1)}
\sum_{i\ne j}
\frac{1}{2}
\left[
r(y_i)s_\theta(y_i)
-
r(y_j)s_\theta(y_i)
+
r(y_j)s_\theta(y_j)
-
r(y_i)s_\theta(y_j)
\right].
\]
Finally, factoring we have
\[
g_{\mathrm{RLOO}}(\theta)
=
\frac{1}{n(n-1)}
\sum_{i\ne j}
\frac{1}{2}
\left(r(y_i)-r(y_j)\right)
\left(s_\theta(y_i)-s_\theta(y_j)\right).
\]
as desired.
\end{proof}

\subsubsection{Computing pair probability under the lattice QMC construction.}

\paragraph{Preliminaries.}
Let \(\mathcal V^\ast\) denote the set of finite token sequences. Arithmetic
sampling defines a deterministic decoder
\[
\Dec_0:[0,1)\to\mathcal V^\ast,
\]
where \(\Dec_0(u)\) is the sequence obtained by running arithmetic decoding
with latent \(u\in[0,1)\). For each sequence \(y\in\mathcal V^\ast\), define
its arithmetic interval
\[
I_0(y):=\{u\in[0,1):\Dec_0(u)=y\}.
\]
Equivalently,
\[
I_0(y)=[L(y),L(y)+W(y))\subset[0,1),
\]
where
\[
W(y)=\pi_\theta(y)=\prod_{t=1}^{T}\pi_\theta(y_t\mid y_{<t}).
\]

For the pair-probability calculation, it will be convenient to identify
\([0,1)\) with the unit circle
\[
\mathbb S^1=\mathbb R/\mathbb Z.
\]
Under this identification, addition is modulo one, and \(I_0(y)\) is viewed as
a subset \(I(y)\subset\mathbb \mathbb S^1\). We will write the resulting 
decoder simply as
\[
\Dec:\mathbb S^1 \to\mathcal V^\ast.
\]

\begin{lemma}[Pair probability as a circular overlap]
\label{lem:pair-prob-circular-overlap}
Let \(b\sim\mathrm{Unif}(\mathbb S^1)\), and suppose the lattice QMC latents are
\[
u_k=b+\alpha_k \pmod 1,
\]
where \(\alpha_k\in\mathbb S^1\) is deterministic. For \(k\in\{i,j\}\), let
\(y_k\in\mathcal V^\ast\) be a sequence of interest, and define
\[
I_k:=I(y_k).
\]
Define the shifted interval
\[
J_k
:=
I_k-\alpha_k
:=
\{b\in\mathbb S^1:\ b+\alpha_k\in I_k\}.
\]
Then
\[
\begin{aligned}
q_{ij}(y_i,y_j)
&:=
\Pr\!\left(
\Dec(b+\alpha_i)=y_i,\;
\Dec(b+\alpha_j)=y_j
\right)= \lambda_{\mathbb S^1}(J_i\cap J_j),
\end{aligned}
\]
where \(\lambda_{\mathbb S^1}\) denotes normalized Lebesgue measure on
\(\mathbb S^1\).
\end{lemma}

\begin{proof}
By definition of \(I_i\),
\[
\Dec(b+\alpha_i)=y_i
\qquad\Longleftrightarrow\qquad
b+\alpha_i\in I_i.
\]
By definition, \(J_i=I_i-\alpha_i\), so the above is equivalent to
\[
b\in J_i.
\]
Similarly,
\[
\Dec(b+\alpha_j)=y_j
\qquad\Longleftrightarrow\qquad
b\in J_j.
\]
Therefore the joint event is exactly
\[
\{b:\Dec(b+\alpha_i)=y_i,\ \Dec(b+\alpha_j)=y_j\}
=
J_i\cap J_j.
\]
Since \(b\) is uniform on \(\mathbb S^1\), the probability of this event is
\[
q_{ij}(y_i,y_j)
=
\lambda_{\mathbb S^1}(J_i\cap J_j).
\]
\end{proof}

Thus, pair probabilities in a QMC group are determined by overlaps among the shifted intervals. 

\paragraph{Numerically stable calculation}

Lemma~\ref{lem:pair-prob-circular-overlap} shows that the pair probability is
the length of the overlap \(J_i\cap J_j\), where
\[
J_k=I_k-\alpha_k.
\]
Naively, this involves materializing the arithmetic intervals which is numerically unstable since it requires computing the raw sequence probabilities. 
We now describe a numerically stable alternative.
A useful way to think about how to compute this overlap is through how much we can perturb this shared
random shift while maintaining the same two decoded sequences.

Let \(b_0\) be the realized shared shift that generated the observed sequences.
Thus
\[
u_k=b_0+\alpha_k \pmod 1,
\qquad
u_k\in I_k.
\]

Now consider a candidate shifted value \(b_0+\delta\). For sample \(k\), this
perturbation keeps the decoded sequence equal to \(y_k\) precisely when
\[
b_0+\delta\in J_k.
\]
Equivalently, adding back the offset \(\alpha_k\),
\[
u_k+\delta\in I_k.
\]
Thus \(J_k-b_0\) is exactly the set of perturbations \(\delta\) of the shared
shift that keep sample \(k\) decoding to \(y_k\).

By translation invariance of Lebesgue measure on \(\mathbb S^1\), we may center
the calculation at \(b_0\):
\[
\lambda_{\mathbb S^1}(J_i\cap J_j)
=
\lambda_{\mathbb S^1}\bigl((J_i-b_0)\cap(J_j-b_0)\bigr).
\]
Moreover, without loss of generality, we can assume the shared anchor is at \(0\), and the problem can be reduced to computing which perturbations \(\delta\)  keep both decoded sequences the same.

The residual coordinate gives these allowable perturbations. For the
\(k\)th sample, define
\[
z_k=\frac{u_k-L_k}{W_k}\in[0,1).
\]
Equivalently,
\[
u_k=L_k+W_kz_k.
\]
Thus \(z_k\) is the fractional position of the sampled latent inside its final
arithmetic interval
\[
I_k=[L_k,L_k+W_k).
\]
It determines the left and right slacks of \(u_k\) within \(I_k\):
\[
\ell_k:=u_k-L_k=W_kz_k,
\qquad
r_k:=L_k+W_k-u_k=W_k(1-z_k).
\]
In words, \(\ell_k\) is how far we can move \(u_k\) to the left before leaving
\(I_k\), and \(r_k\) is how far we can move \(u_k\) to the right before leaving
\(I_k\).

Therefore the perturbation that keep sample \(k\) fixed are thes ones where \(u_k+\delta\) remains inside \(I_k\). 
This is exactly when
\[
\delta\in[-\ell_k,r_k)\pmod 1.
\]

Since this condition must be satisfied so that both samples $i$ and $j$ remain fixed, we obtain
\[
q_{ij}(y_i,y_j)
=
\lambda_{\mathbb S^1}
\left(
[-\ell_i,r_i)\cap[-\ell_j,r_j)
\right),
\]
Again, the intersection is understood on the circle.

\begin{proposition}[Stable formula for the circular overlap]
\label{prop:stable-circular-overlap}
% For \(k\in\{i,j\}\), let
% \[
% \ell_k=W_kz_k,
% \qquad
% r_k=W_k(1-z_k).
% \]
% Then
% \[
% \boxed{
% q_{ij}(y_i,y_j)
% =
% \min(\ell_i,\ell_j)
% +
% \min(r_i,r_j)
% +
% (r_i+\ell_j-1)_+
% +
% (\ell_i+r_j-1)_+
% }
% \]
% where \((a)_+=\max\{a,0\}\).
\end{proposition}

\begin{proof}
By the calculation above,
\[
q_{ij}(y_i,y_j)
=
\lambda_{\mathbb S^1}(K_i\cap K_j),
\]
where
\[
K_i=[-\ell_i,r_i)\pmod 1,
\qquad
K_j=[-\ell_j,r_j)\pmod 1.
\]
Cut the circle at \(0\). Then,
\[
K_k=[0,r_k)\cup[1-\ell_k,1).
\]
The overlap has four possible contributions.

The first two conditions are best understood by considering how two sub-intervals within the unit interval can intersect.
First, the two right pieces overlap by
\[
\lambda\bigl([0,r_i)\cap[0,r_j)\bigr)=\min(r_i,r_j).
\]
Second, the two left pieces overlap by
\[
\lambda\bigl([1-\ell_i,1)\cap[1-\ell_j,1)\bigr)
=
\min(\ell_i,\ell_j).
\]

The last conditions are more subtle and occur when there is wrap-around. 

First, the right piece of \(K_i\) can overlap the left piece of \(K_j\):
\[
[0,r_i)\cap[1-\ell_j,1).
\]
This intersection is nonempty exactly when \(r_i+\ell_j>1\). In that case its
length is \(r_i+\ell_j-1\). Thus its contribution is
\[
(r_i+\ell_j-1)_+.
\]
Similarly, the left piece of \(K_i\) can overlap the right piece of \(K_j\),
contributing
\[
(\ell_i+r_j-1)_+.
\]
Adding the four contributions gives
\[
q_{ij}(y_i,y_j)
=
\min(\ell_i,\ell_j)
+
\min(r_i,r_j)
+
(r_i+\ell_j-1)_+
+
(\ell_i+r_j-1)_+.
\]

\end{proof}

% ---- colors (seaborn deep) ----
\definecolor{colI}{HTML}{C44E52} % K_i
\definecolor{colJ}{HTML}{4C72B0} % K_j
\definecolor{colO}{HTML}{55A868} % overlap
 
% ---- dimensions ----
\newcommand{\ovRI}{0.78}\newcommand{\ovRJ}{0.95}
\newcommand{\ovrin}{0.60}\newcommand{\ovRC}{1.06}
\newcommand{\ovLW}{3.0}\newcommand{\ovbh}{0.13}
\newcommand{\ovhI}{0.10}\newcommand{\ovhJ}{0.34}\newcommand{\ovbarTop}{0.50}
 
% ---- circle panel: #1=li #2=ri #3=lj #4=rj #5=hs #6=he ----
\newcommand{\ovCircle}[6]{%
  \fill[colO,opacity=0.38]
    ({360*#5}:\ovrin) arc[start angle={360*#5},end angle={360*#6},radius=\ovrin]
    -- ({360*#6}:\ovRC) arc[start angle={360*#6},end angle={360*#5},radius=\ovRC] -- cycle;
  \draw[colO!75,line width=0.4pt] ({360*#5}:\ovrin)--({360*#5}:\ovRC);
  \draw[colO!75,line width=0.4pt] ({360*#6}:\ovrin)--({360*#6}:\ovRC);
  \draw[gray!45,line width=0.5pt] (0,0) circle (\ovRC);
  \draw[colI,line width=2.4pt,line cap=round]
    ({-360*#1}:\ovRI) arc[start angle={-360*#1},end angle={360*#2},radius=\ovRI];
  \draw[colJ,line width=2.4pt,line cap=round]
    ({-360*#3}:\ovRJ) arc[start angle={-360*#3},end angle={360*#4},radius=\ovRJ];
  \draw[gray!75] (0:\ovRC)--(0:{\ovRC+0.07});
  \node[font=\scriptsize,inner sep=1pt,anchor=west] at (0:{\ovRC+0.05}) {$0$};
}
 
% ---- line panel: same args ----
\newcommand{\ovLine}[6]{%
  \fill[colO,opacity=0.32] ({#5*\ovLW},-0.03) rectangle ({#6*\ovLW},\ovbarTop);
  \draw[gray!55,line width=0.5pt] (0,0)--(\ovLW,0);
  \draw[gray!55] (0,0.05)--(0,-0.05) (\ovLW,0.05)--(\ovLW,-0.05);
  \node[font=\scriptsize,anchor=north] at (0,-0.06) {$0$};
  \node[font=\scriptsize,anchor=north] at (\ovLW,-0.06) {$1$};
  \fill[colI] (0,\ovhI) rectangle ({#2*\ovLW},{\ovhI+\ovbh});
  \fill[colI] ({(1-#1)*\ovLW},\ovhI) rectangle (\ovLW,{\ovhI+\ovbh});
  \fill[colJ] (0,\ovhJ) rectangle ({#4*\ovLW},{\ovhJ+\ovbh});
  \fill[colJ] ({(1-#3)*\ovLW},\ovhJ) rectangle (\ovLW,{\ovhJ+\ovbh});
}
 
% ============================================================
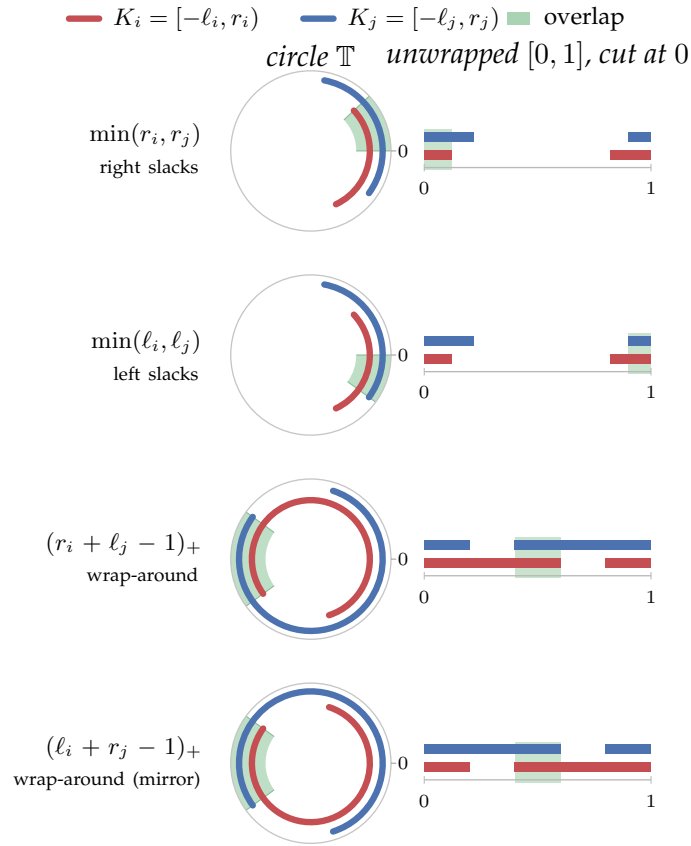
\begin{figure}[t]
  \centering
  \begin{tikzpicture}[font=\small]
 
    % legend
    \begin{scope}[yshift=1.75cm]
      \draw[colI,line width=2.4pt,line cap=round] (-3.2,0)--(-2.8,0);
      \node[anchor=west,font=\footnotesize] at (-2.75,0) {$K_i=[-\ell_i,r_i)$};
      \draw[colJ,line width=2.4pt,line cap=round] (-0.1,0)--(0.3,0);
      \node[anchor=west,font=\footnotesize] at (0.35,0) {$K_j=[-\ell_j,r_j)$};
      \fill[colO,opacity=0.5] (2.6,-0.08) rectangle (2.9,0.08);
      \node[anchor=west,font=\footnotesize] at (2.95,0) {overlap};
    \end{scope}
 
    % column headers
    \node[font=\itshape] at (0,1.26) {circle $\mathbb{T}$};
    \node[font=\itshape] at ({1.5+\ovLW/2},1.26) {unwrapped $[0,1]$, cut at $0$};
 
    % Row 1 : no wrap, min(r)
    \begin{scope}[yshift=0cm]
      \node[anchor=east,align=right,font=\footnotesize,text width=2.3cm] at (-1.35,0)
        {$\displaystyle\min(r_i,r_j)$\\[2pt]\scriptsize right slacks};
      \begin{scope} \ovCircle{0.18}{0.12}{0.10}{0.22}{0}{0.12} \end{scope}
      \begin{scope}[xshift=1.5cm,yshift=-0.22cm] \ovLine{0.18}{0.12}{0.10}{0.22}{0}{0.12} \end{scope}
    \end{scope}
 
    % Row 2 : no wrap, min(l)
    \begin{scope}[yshift=-2.7cm]
      \node[anchor=east,align=right,font=\footnotesize,text width=2.3cm] at (-1.35,0)
        {$\displaystyle\min(\ell_i,\ell_j)$\\[2pt]\scriptsize left slacks};
      \begin{scope} \ovCircle{0.18}{0.12}{0.10}{0.22}{0.90}{1.0} \end{scope}
      \begin{scope}[xshift=1.5cm,yshift=-0.22cm] \ovLine{0.18}{0.12}{0.10}{0.22}{0.90}{1.0} \end{scope}
    \end{scope}
 
    % Row 3 : wrap (r_i + l_j - 1)+
    \begin{scope}[yshift=-5.4cm]
      \node[anchor=east,align=right,font=\footnotesize,text width=2.5cm] at (-1.35,0)
        {$\displaystyle (r_i+\ell_j-1)_+$\\[2pt]\scriptsize wrap-around};
      \begin{scope} \ovCircle{0.2}{0.6}{0.6}{0.2}{0.4}{0.6} \end{scope}
      \begin{scope}[xshift=1.5cm,yshift=-0.22cm] \ovLine{0.2}{0.6}{0.6}{0.2}{0.4}{0.6} \end{scope}
    \end{scope}
 
    % Row 4 : wrap (l_i + r_j - 1)+  (mirror)
    \begin{scope}[yshift=-8.1cm]
      \node[anchor=east,align=right,font=\footnotesize,text width=2.5cm] at (-1.35,0)
        {$\displaystyle (\ell_i+r_j-1)_+$\\[2pt]\scriptsize wrap-around (mirror)};
      \begin{scope} \ovCircle{0.6}{0.2}{0.2}{0.6}{0.4}{0.6} \end{scope}
      \begin{scope}[xshift=1.5cm,yshift=-0.22cm] \ovLine{0.6}{0.2}{0.2}{0.6}{0.4}{0.6} \end{scope}
    \end{scope}
 
  \end{tikzpicture}
  \caption{%
The pairwise probability $q_{ij}=\lambda_{\mathbb T}(K_i\cap K_j)$ decomposes into four terms by
Proposition~\ref{prop:stable-circular-overlap}.    
In each row, we visualize the intersections, showing both the intersection on the circle (left) and the corresponding intersections on the unit interval (arbitrarily) cut at $0$ (right).
Rows~1--2 show the two pieces of the overlap adjacent to the cut at $0$, with lengths $\min(r_i,r_j)$ and $\min(\ell_i,\ell_j)$.
Rows~3--4 show the additional ``cross'' intersections that appear when one
interval reaches far enough to meet the opposite side of the other interval,
with lengths $(r_i+\ell_j-1)_+$ and $(\ell_i+r_j-1)_+$.
}
  \label{fig:circular-overlap}
\end{figure}
 

%%%%%%%%%%%%%%%%%%%%%%%%%%%%%%%%%%%%%%%%%%%%%%%%%%%%%%%%%%%%

\newpage

\end{document}